\documentclass{article}

\usepackage{microtype}
\usepackage{graphicx}
\usepackage{subcaption}
\usepackage{booktabs} % for professional tables
\usepackage{makecell} % for \Xhline
\usepackage{wrapfig}
\usepackage{hyperref}
\usepackage{xcolor} 
\usepackage{pifont}
\usepackage{subcaption}
\usepackage[preprint]{icml2026}

\usepackage{amsmath}
\usepackage{amssymb}
\usepackage{mathtools}
\usepackage{amsthm}
\usepackage{colortbl} 
\usepackage{multirow}
\usepackage[capitalize,noabbrev]{cleveref}

\theoremstyle{plain}
\newtheorem{theorem}{Theorem}[section]

\newtheorem{lemma}[theorem]{Lemma}

\theoremstyle{definition}

\theoremstyle{remark}
\newtheorem{remark}[theorem]{Remark}

\usepackage[textsize=tiny]{todonotes}

\icmltitlerunning{Exploring Multiple High-Scoring Subspaces in Generative Flow Networks}

\begin{document}

\twocolumn[
  \icmltitle{Exploring Multiple High-Scoring Subspaces in Generative Flow Networks}



  \icmlsetsymbol{equal}{*}

  \begin{icmlauthorlist}
    \icmlauthor{Xuan Yu}{y1}
    \icmlauthor{Xu Wang}{equal,y1,y2}
    \icmlauthor{Rui Zhu}{y1}
    \icmlauthor{Yudong Zhang}{y1,y2}
    \icmlauthor{Yang Wang}{equal,y1,y2,y3}
  \end{icmlauthorlist}

  \icmlaffiliation{y1}{University of Science and Technology of China (USTC), Hefei, China}
  \icmlaffiliation{y2}{Suzhou Institute for Advanced Research, USTC, Suzhou, China}
  \icmlaffiliation{y3}{State Key Laboratory of Precision and Intelligent Chemistry, USTC, Hefei, China}

  \icmlcorrespondingauthor{Xu Wang}{wx309@ustc.edu.cn}
  \icmlcorrespondingauthor{Yang Wang}{angyan@ustc.edu.cn}

  \icmlkeywords{Machine Learning, ICML}

  \vskip 0.3in
]


\printAffiliationsAndNotice{}  % no special notice (required even if empty)
\newcommand{\modelname}{{CMAB-GFN}}
\newcommand{\randmodel}{{RandGFN}}
\newcommand{\GFNs}{{GFlowNets}}
\newcommand{\GFN}{{GFlowNet}}

\begin{abstract}
  As a probabilistic sampling framework, Generative Flow Networks (GFlowNets) show strong potential for constructing complex combinatorial objects through the sequential composition of elementary components. However, existing GFlowNets often suffer from excessive exploration over vast state spaces, leading to over-sampling of low-reward regions and convergence to suboptimal distributions. Effectively biasing GFlowNets toward high-reward solutions remains a non-trivial challenge. In this paper, we propose CMAB-GFN, which integrates a combinatorial multi-armed bandit (CMAB) framework with GFlowNet policies. The CMAB component prunes low-quality actions, yielding compact high-scoring subspaces for exploration. Restricting GFNs to these compact high-scoring subspaces accelerates the discovery of high-value candidates, while the exploration of different subspaces ensures that diversity is not sacrificed. Experimental results on multiple tasks demonstrate that CMAB-GFN generates higher-reward candidates than existing approaches.
\end{abstract}

% \begin{abstract}
%   As a probabilistic sampling framework, Generative Flow Networks (GFNs) show strong potential for constructing complex combinatorial objects through the sequential composition of elementary components. However, existing GFNs often suffer from excessive exploration over vast state spaces, leading to over-sampling of low-reward regions and convergence to suboptimal distributions. Effectively biasing GFNs toward high-reward solutions remains a non-trivial challenge. In this paper, we propose CMAB-GFN, which integrates a combinatorial multi-armed bandit (CMAB) framework with GFN policies. The CMAB component prunes low-quality actions, yielding compact high-scoring subspaces for exploration. Restricting GFNs to these compact high-scoring subspaces accelerates the discovery of high-value candidates, while the exploration of different subspaces ensures that diversity is not sacrificed. Experimental results on multiple tasks demonstrate that CMAB-GFN generates higher-reward candidates than existing approaches, without sacrificing diversity. All implementations are publicly available at \url{https://anonymous.4open.science/r/CBFlowNet-E0BA/}.
% \end{abstract}
\section{Introduction}

% gflownet背景
Generative Flow Networks (GFlowNets)~\citep{bengio2021flow,zhang2022generative,cretu2024synflownet} have shown their impressive potential in generating diverse and high-scoring candidates across various domains, especially in generating combinatorial objects~\citep{zhang2023distributional,zhang2023let}. By unifying MDPs’ sequential dynamics with flow-based probability matching, GFlowNets synthesize action trajectories that sample candidates proportionally to the desired reward distribution. 

% issue
Despite this appealing objective, practical GFlowNet training often suffers from inefficient exploration in large combinatorial spaces~\citep{kim2023local,shen2023towards}. In many realistic tasks, the vast majority of states yield negligible reward, while high-reward candidates occupy a vanishingly small fraction of the space. As a result, early and mid-stage training frequently resembles near-uniform exploration over actions, causing most sampled trajectories to terminate in low-value regions. Under limited sampling budgets, this leads to slow discovery of high-reward modes and empirical terminal distributions that underrepresent the most valuable solutions, even when the theoretical objective is reward-proportional sampling.

% 现有方法
Several approaches attempt to increase the \emph{greediness} of GFlowNets, thus biasing to high-reward regions. Temperature scaling modifies the target distribution to $R(x)^\beta$ with $\beta \gg 1$, but choosing $\beta$ introduces instability and risks mode collapse \citep{malkin2022trajectory,lau2024qgfn}. Other methods incorporate Q-values \citep{lau2024qgfn}, local search \citep{kim2023local}, or evolutionary strategies \citep{kim2024genetic}, which bias sampling toward promising regions. However, these techniques typically operate at the \emph{local action level} and pay more attention to exploiting observed high-value regions. Once certain actions are deemed low-value, they are rarely revisited. This may irreversibly suppress actions that appear weak early but are essential for later high-reward compositions. Consequently, exploration can become trapped in a narrow region of the space.

\begin{figure}[t]
    \centering
    \includegraphics[width=0.98\linewidth]{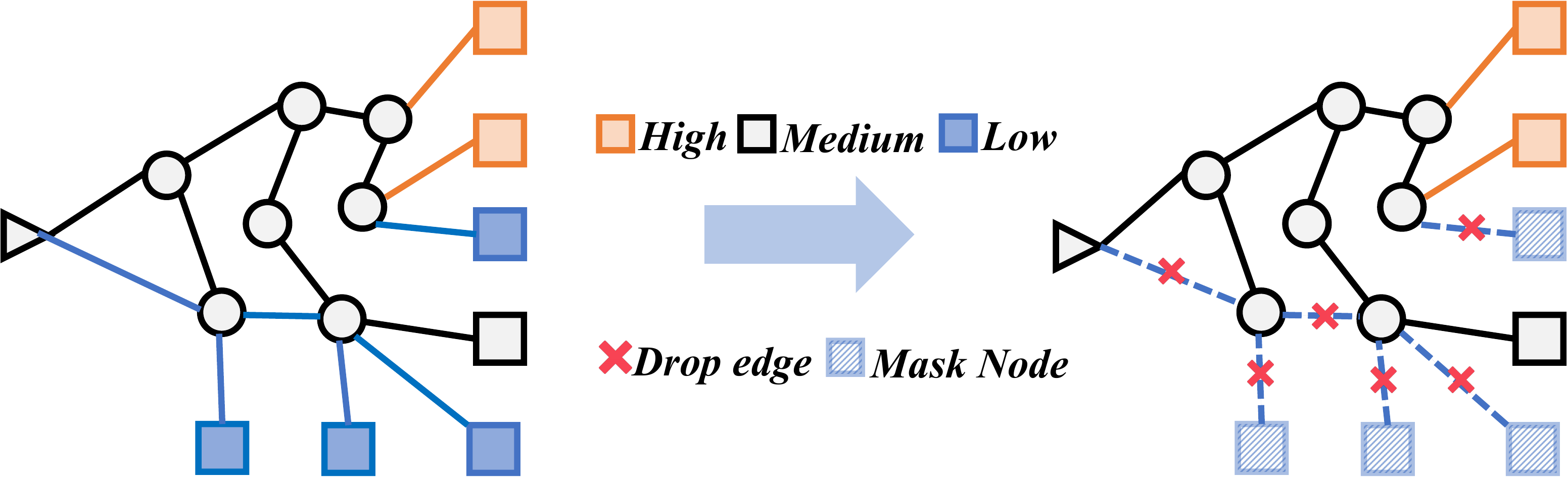}
    \caption{\textbf{Illustration of action pruning}.The triangle means initial state, the circles denote interior states, and the squares denote the terminal states. By pruning low-scoring actions (blue edges), candidates with low rewards(blue nodes) are masked. Candidates with high rewards (Orange ones) are more likely to be explored, addressing the over-exploration of low-reward candidates.}
    \label{fig:prune}
\end{figure}

% observation
Inspired by existing works, we observe that in structured generative domains certain actions are consistently more promising than others. For example, in molecule design, combining specific blocks tends to yield high-scoring candidates more frequently. This suggests a different perspective on exploration: rather than searching uniformly over the entire state space, one may instead search over \emph{subspaces induced by subsets of actions}. This motivates a key intuition that even after aggressively pruning actions, the remaining subspace is still vast, yet may contain a much higher density of high-reward candidates. As shown in Fig.~\ref{fig:prune}, by pruning low-quality actions and retaining promising ones, we can substantially improve sampling efficiency. Based on this, we reformulate GFlowNet exploration as a \emph{higher-level combinatorial decision problem}. Instead of sampling over the full action space at all times, we repeatedly select promising subsets of state-independent actions that define compact, high-scoring subspaces, and train the GFlowNet within those subspaces. Exploration thus occurs at the subspace level rather than solely at the trajectory level.

Formally, we define the pruning process as selecting $K$ actions from a total of $N$ actions to retain, thereby inducing a subspace of the original state space. Different pruning choices correspond to different subspaces, some of which contain significantly denser clusters of high-reward candidates. Importantly, our objective is not to identify a single optimal subspace, but to discover diverse high-scoring subspaces for exploration. Continuously sampling from the same subspace would inevitably increase the similarity of generated candidates, quantitative evidence for this effect is provided in the experiments.

However, identifying such promising subspaces is non-trivial. When $K$ scales with $N$, the number of possible pruning configurations grows combinatorially as $\binom{N}{K}$, rendering exhaustive search infeasible. Under a limited sampling budget, the problem naturally exhibits an exploration–exploitation trade-off: one must decide whether to exploit pruning strategies that have yielded high rewards so far, or to explore alternative pruning configurations that may uncover even better subspaces. This setting can be viewed through the lens of a multi-armed bandit (MAB) problem, where each arm corresponds to an action and rewards are observed only through downstream sampling~\citep{robbins1952some}. Since each pruning decision involves selecting a subset of actions rather than a single action, the problem more precisely aligns with the combinatorial multi-armed bandit (CMAB) framework~\citep{chen2013combinatorial}, where the learner selects combinations of base arms and receives bandit feedback on the resulting base arms. We combine a CMAB framework with GFlowNets, and introduce {\modelname}. By considering actions as base arms in the CMAB problem, we can utilize the CMAB algorithm to select actions that are more likely to lead to high-reward candidates. 

Considering that existing works~\citep{kim2023local,shen2023towards} that increase GFlowNet greediness through local value thresholding often suffer from exploration collapse, our method does not rely on local value thresholding to restrict actions. Instead, we partition the original search space into subspaces defined by super arms and explore them under a CMAB framework. These subspaces are continuously cycled and re-evaluated, ensuring that no region becomes persistently inaccessible. Furthermore, our algorithm explicitly enforces subspace-level exploration, preventing the sampler from collapsing into a single high-scoring subspace. Conceptually, this reflects a shift in perspective: rather than operating directly over the entire state space, we view GFlowNet sampling as a process of discovering high-reward subspaces among a large collection of candidates. This formulation naturally balances exploration and exploitation at the subspace level, avoiding the effective lock-in behavior induced by hard thresholding, and leading to more robust performance.

We evaluate the proposed \modelname\ on several popular tasks used in prior works, including molecule design~\citep{bengio2021flow}, three RNA design tasks~\citep{sinai2020adalead} and bit sequence task~\citep{malkin2022trajectory}. The result demonstrates that the proposed method discovers more high-reward candidates and converges faster than baselines. Our main contributions are:

\noindent \ding{182} We propose a new perspective that views GFlowNet exploration as a combinatorial search over action-induced subspaces rather than uniform exploration of the full state space.

\noindent \ding{183} A novel framework CMAB-GFN is proposed, which integrates combinatorial multi-armed bandits with GFlowNet training to adaptively select high-value subspaces.
    % \item We design a two-phase reward estimation protocol that enables bandit learning under the non-stationary dynamics of evolving GFlowNet policies.

\noindent \ding{184} We demonstrate through extensive experiments that CMAB-GFN substantially improves high-reward mode discovery across molecular, RNA, and synthetic sequence design tasks.

\section{Related Work}
\textbf{GFlowNets:}
Since their introduction by \citet{bengio2021flow}, GFlowNets have advanced rapidly in theory and applications, with recent works establishing connections to variational inference \citep{malkin2022gflownets}, distributional analysis \citep{silva2025gflownets}, proxy-free training in offline settings \citep{chen2025proxy}, and alternative loss designs \citep{hu2024beyond}. They have also been applied to combinatorial optimization tasks, including general problems \citep{zhang2023let}, computation graphs \citep{zhang2023robust}, hierarchical exploration with evolutionary search \citep{kim2024ant}, and multi-objective optimization \citep{zhu2023sample}. In contrast to these approaches, which mainly employ GFlowNets to solve combinatorial problems, our method leverages combinatorial optimization techniques to improve GFlowNets themselves. Complementary efforts have enhanced GFlowNet training through Q-value integration \citep{lau2024qgfn}, local search \citep{kim2023local}, Thompson sampling \citep{rector2023thompson}, replay strategies \citep{vemgal2023empirical,shen2023towards}, genetic and evolutionary algorithms \citep{kim2024genetic,ikram2024evolution}, and adaptive teacher mechanisms \citep{kim2024adaptive}.

\textbf{Combinatorial multi-armed bandit:}
The combinatorial multi-armed bandit (CMAB) framework was first introduced by \citep{chen2013combinatorial}. \citet{chen2016combinatorial} later extended it to nonlinear reward functions dependent on variable distributions. Subsequent work includes the cost-aware auction-based CMAB by \citet{gao2021auction} and full-bandit algorithms by \citet{agarwal2021dart,fourati2024combinatorial}, where no individual arm feedback is available.

\section{Preliminary}
In the classic Combinatorial Multi-Armed Bandit (CMAB) problem~\citep{chen2013combinatorial,chen2016combinatorial}, there are $N$ base arms, each associated with an unknown reward distribution. In each round, the player selects a subset of $K$ base arms, forming a super arm, and receives a joint reward. The goal is to identify the best $K$ arms that maximize the joint reward, which may be a non-linear function of individual arm rewards, while minimizing regret—the gap between the expected reward of always playing the optimal super arm and that of the algorithm’s choices~\citep{chen2013combinatorial, slivkins2019introduction,chen2016combinatorial}. The central challenge lies in balancing exploration (trying diverse super arms to gather information) and exploitation (selecting the current best super arm for higher reward). Notably, the flow network embodies a similar dilemma: some regions of the state space contain dense clusters of high-reward candidates, and the algorithm must decide between probing new promising subspaces and exploiting those already identified.
The feedback models of CMAB can be categorized into two types: 
1) Full-bandit feedback\citep{chen2013combinatorial,fourati2024federated}, where only the aggregate reward of the played super arm is observed, and
2) Semi-bandit feedback\citep{chen2013combinatorial,chen2021combinatorial}, where the individual rewards of each base arm in the super arm are additionally revealed.
Our problem uses \emph{arm-level} feedback in the spirit of the semi-bandit setting: from evaluation samples, we construct per-arm statistics that guide super-arm selection. Importantly, we do \emph{not} observe a physically separate stochastic outcome for each arm as in classical semi-bandits. Instead, arm feedback is a \emph{proxy} derived from terminal rewards (Sec.~\ref{reward disgn}).
\section{Methodology}
In this section, we introduce \textbf{Combinatorial Multi-Armed Bandit GFlowNet ({\modelname})}, a greedy training framework designed to enhance both the quality and diversity of generated candidates based on CMAB. Unlike prior approaches that operate over the entire flow graph~\citep{lau2024qgfn}, {\modelname} achieves a more balanced exploration--exploitation trade-off by selectively focusing on high-scoring subspaces of substantially reduced size.

\subsection{Framework Design}

\subsubsection{Design of Base and Super Arms}
The core challenge in applying CMAB to GFlowNets is to define a set of base arms whose combinations (super arms) can meaningfully constrain the exploration space. A naive approach would be to treat every possible state transition $s\rightarrow s'$ as a distinct arm. However, this leads to an intractably large and state-dependent set of arms, making the CMAB problem ill-posed.

To overcome this, we observe that in many sequential generation tasks, actions can be decomposed into two components: A \textbf{state-dependent} component $a_d$ that determines where to act (e.g., which position in a sequence to fill, which molecular stem to extend); A \textbf{state-independent} component $a_i$ that determines what action to take, regardless of the specific state (e.g., which value to assign to a position, which building block to attach).
\begin{equation}
\mathcal{A}_i = \{a_i \mid (a_d, a_i) \in \mathcal{A}_s, \forall s \in \mathcal{S}\}.
\end{equation}

$A_s$ denotes the available transitions(action set) of state $s$. Intuitively, $A_i$ represents the "alphabet" of primitive choices available throughout the generative process. A super arm $\mathbb{S} \subseteq A_i $ is then a subset of this alphabet. By selecting a super arm, we restrict the GFlowNet policy such that at any state $s$, it can only take actions $(a_d,a_i)$ where $a_i \in \mathbb{S}$. This effectively prunes all actions that use primitive choices outside $\mathbb{S}$. Task Examples are given in Table~\ref{tab:action-decomposition}.

\begin{table*}[t]
\centering
\caption{Characteristics of four evaluation tasks from three perspectives.}
\resizebox{\textwidth}{!}{
    \begin{tabular}{lccc}
    \Xhline{1.0pt}
    \rowcolor{lightgray!20} \textbf{Task} & \textbf{State-dependent Component ($a_d$)} & \textbf{State-independent Component ($a_i$, Base Arm)} & \textbf{$\mathcal{A}_i$} \\
    \Xhline{1.0pt}
    Bit Sequence Generation & Position to edit & Binary value to assign & $\{0, 1\}$ \\
    \rowcolor{lightgray!20} Molecule Design         & Stem to extend   & Building block to attach & Vocabulary of 105 blocks \\
    RNA Sequence Design     & Prepend or Append & Nucleotide to add & $\{A, C, G, U\}$ \\
    \Xhline{1.0pt}
    \end{tabular}
}
\label{tab:action-decomposition}
\end{table*}

\subsubsection{Design of Rewards for Base Arms}
\label{reward disgn}

We now turn to the assignment of rewards to base arms, a critical factor for ensuring stable learning.  
Ideally, to evaluate the quality of a base arm $i$, we define its expected reward as the average reward across \emph{all} states in the entire state space that contain this base arm:
\begin{equation}
\mu_i = \mathbb{E}_{x \in \mathcal{X}_i}[r(x)],
\end{equation}
where $\mathcal{X}_i$ denotes the set of all possible states containing base arm $i$, and $r(x)$ is the normalized reward of state $x$. This definition aligns with the CMAB assumption that the reward distribution of a base arm is intrinsic to the arm itself and remains invariant across different super arms.

In practice, however, we cannot enumerate the entire state space $\mathcal{X}_i$ due to its exponential size. Instead, we estimate the reward of base arm $i$ using only the sampled candidates. In the semi-bandit setting with $m$ base arms, the reward of base arm $i$ at round $t$ is approximated as
\begin{equation}
X_i^t = \frac{1}{|C_i^t|}\sum_{x \in C_i^t} r(x),
\label{eq-arm-reward}
\end{equation}
where $C_i^t$ denotes the set of candidates \emph{sampled} at round $t$ that contain base arm $i$, and
\begin{equation}
r(x) = \text{normalize}(R(x))
\end{equation}
is the normalized reward of candidate $x$ with raw reward $R(x)$ from the environment. In our implementation, we use \emph{global min--max normalization} to constrain rewards into $[0,1]$:
\begin{equation}
r(x)=\mathrm{clip}\Bigg(\frac{R(x)-R_{\min}}{R_{\max}-R_{\min}+\epsilon}, 0, 1\Bigg),
\end{equation}
where $R_{\min}$ and $R_{\max}$ are the global minimum and maximum raw rewards observed so far (updated online during sampling), and $\epsilon$ is a small constant for numerical stability.

To ensure that our empirical estimate $X_i^t$ converges to the true expected reward $\mu_i$, we adopt a two-phase sampling strategy, which \emph{asymptotically} preserves the distributional consistency required by CMAB as the flow network accumulates diverse training experience (Theorem~\ref{thm:2}):

% \begin{itemize}[leftmargin=2em,itemsep=-0.1em]

\noindent \ding{182} \textbf{Constrained training.} Train the flow network restricted to $\mathbb{S}$.

\noindent \ding{183} \textbf{Unrestricted evaluation.} Sample candidates without restrictions, and compute $X_i$ as the average reward over all candidates containing arm $i$.

This procedure increases the sampling cost but does not increase network training complexity, while asymptotically ensuring unrestricted reward estimates as the flow network accumulates diverse training experience. Moreover, the extra cost does not scale linearly with training time, since the evaluation stage of deep networks remains unaffected. With multi-threaded sampling, the overhead can be further reduced. A detailed comparison of time consumption is reported in Table~\ref{tab:timecost}. It is important to note that the samples generated in the unrestricted evaluation phase are used only for updating bandit statistics and are \emph{not} included in the final evaluation metrics reported in our experiments. 

\subsubsection{Design for Adjusting Network Structure}\label{sec:adjust-network}
We further propose a method to address the challenge that arises when the independent action space is small but trajectories are long, resulting in narrow yet deep networks.

\begin{figure}[t]
    \centering
    \includegraphics[width=0.98\linewidth]{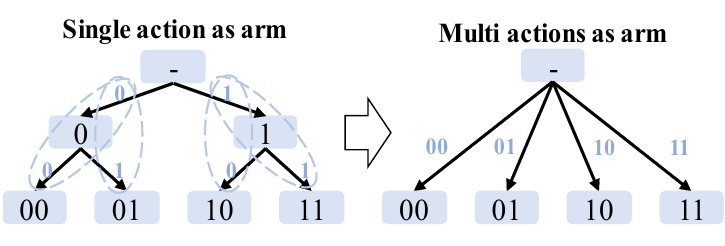}
    \caption{Using short action sequences as arms to transform the narrow-deep network architecture into a more balanced wide structure.}
    \vspace{-15pt}
    \label{fig:multi-actions}
\end{figure}
For instance, in bit sequence generation, there are only two independent actions, ${0,1}$, while a complete sequence may require over 100 steps. This leads to a deep but fragile network, where pruning even a single action can collapse the entire solution space.

To address this, we redefine base arms as short sequences of $t$ consecutive actions:
\[
a_1 \rightarrow a_2 \rightarrow \cdots \rightarrow a_t.
\]
Super arms then consist of sets of such sequences. A sub-trajectory is valid if its independent subsequence belongs to the chosen super arm.  
This widens the effective search space and balances the architecture, while keeping the CMAB formulation applicable (see Fig.~\ref{fig:multi-actions}). 

\subsubsection{Dynamic Nature of Flow Networks}
Unlike standard CMAB problems with stationary distributions, flow networks evolve during training. It is therefore important to characterize this non-stationarity. The fundamental constraint for an ideal flow network is:
\begin{equation}
    \pi(x) = \frac{R(x)}{\sum_{x' \in \mathcal{X}} R(x')}, \quad \forall x \in \mathcal{X},
\end{equation}
where \( \pi(x) \) represents the target distribution and \( R(x) \) denotes the reward function. However, achieving this equilibrium condition requires exhaustive exploration of all states, which is impractical in real-world scenarios. Consequently, the expected rewards of individual base arms evolve dynamically throughout the training process of the flow network. Nevertheless, under a mild policy-stabilization condition (i.e., when the GFlowNet parameters eventually converge or change sufficiently slowly), the induced arm reward distributions stabilize in late training.
In our analysis, we therefore separate two aspects: (i) \emph{TB consistency}---if the TB equations are solved (or nearly solved), the induced terminal distribution matches (or approximates) the reward-proportional target; and (ii) \emph{training non-stationarity}---during optimization, the policy drifts, so arm reward statistics can be non-stationary. Our convergence statements below are correspondingly stated under a policy-stabilization condition, and in practice we use sliding-window estimates to track residual drift.

\textbf{Theorem 1 (Convergence of Reward Distributions).}
Under standard regularity conditions (bounded rewards, positive exploration, appropriate learning rates, and finite state space), the reward distribution of each base arm $i$ converges to a stable distribution as the GFlowNet policy converges. The formal statement with precise assumptions and the complete proof are provided in Appendix~\ref{Proof 1}.

\textbf{Theorem 2 (Asymptotic Approximate Distributional Invariance).}
\label{thm:2}
Under the two-phase protocol, if the evaluation-time conditional sampling distributions become asymptotically insensitive to which super arm was used in the preceding restricted training phase (e.g., due to diminishing policy drift / two-time-scale updates), then the arm-level reward statistics used by CMAB become approximately invariant across super arm selections in late training. The formal condition and proof are provided in Appendix~\ref{Proof 2}.

\subsubsection{CUCB Algorithm for \modelname}
We integrate our framework into the \textit{Combinatorial Upper Confidence Bound (CUCB)} algorithm (Algorithm\ref{alg:CUCB}). To better handle the non-stationary rewards induced by ongoing flow training (Section above), we compute bandit statistics using a \textbf{sliding window} of the most recent $H$ unrestricted evaluation rounds (i.e., the feedback obtained from \textsc{gflownet.gen}($ALL$), aggregated every $I$ rounds if interval aggregation is used). Concretely, for each arm $i$, we maintain a FIFO buffer $\mathcal{B}_i$ that stores up to $H$ most recent observed arm rewards $X_i$ (Eq.~\ref{eq-arm-reward}); we then set $\hat{\mu}_i=\frac{1}{|\mathcal{B}_i|}\sum_{x\in\mathcal{B}_i} x$. The UCB-adjusted estimate is
\begin{equation}
\overline{\mu}_i = \hat{\mu}_i + \sqrt{\tfrac{3\ln t}{2T_i}},
\label{eq-4}
\end{equation}

where $t$ is the round index. This estimate ensures a principled balance between exploration and exploitation. Aside from the CUCB algorithm, there exist several other CMAB algorithms, such as Combinatorial Thompson Sampling (CTS)~\citep{wang2018thompson} and Efficient Sampling for Combinatorial Bandits (ESCB)~\citep{kveton2015tight}. Our framework is algorithm-agnostic and works with other CMAB methods like CTS and ESCB.

% \noindent\textbf{Remark (Non-stationarity).} The sliding-window estimator is a practical mechanism widely used in non-stationary bandits to trade off bias (tracking drift) and variance (finite samples). We do not claim a general convergence guarantee for the joint CMAB--GFlowNet dynamics without additional assumptions (e.g., bounded per-round drift / two-time-scale updates). We empirically study stability w.r.t.\ $H$ in Appendix and treat $H$ as a robustness knob.

\begin{algorithm}[!h]
\caption{Combinatorial Upper Confidence Bound (CUCB) with Flow Network}
\label{alg:CUCB}
\renewcommand{\algorithmicrequire}{\textbf{Maintain:}}
\renewcommand{\algorithmicensure}{\textbf{Inputs:}}
\begin{algorithmic}[1]
\REQUIRE $\mathcal{B}_i$: buffer of recent rewards for arm $i$; $T_i = |\mathcal{B}_i|$: sample count; $\hat{\mu}_i$: empirical mean
\ENSURE $m$: \# base arms; $K$: super arm size; $T$: \# training rounds; $H$: window size
\STATE \textbf{// Initialization:}
\WHILE{$\exists i \in \{1,\dots,m\} \text{ with } T_i = 0$}
    \STATE \textsc{gflownet.train}($ALL$); Receive feedback from \textsc{gflownet.gen}($ALL$)
    \STATE Compute and update arm rewards $\{X_i\}$, $\mathcal{B}_i$, $T_i$, $\hat{\mu}_i$
\ENDWHILE
\STATE \textbf{// Main Loop:}
\FOR{$t = m \rightarrow T$}
    \STATE Compute UCB: $\overline{\mu}_i \gets \hat{\mu}_i + \sqrt{\frac{3\ln t}{2T_i}}$; Select $\mathbb{S} = \text{select}(\overline{\mu}_1,\dots,\overline{\mu}_m)$
    \STATE \textsc{gflownet.train}($\mathbb{S}$); Receive feedback from \textsc{gflownet.gen}($ALL$)
    \STATE Compute and update arm rewards $\{X_i\}_{i\in\mathbb{S}}$, $\mathcal{B}_i$, $T_i$, $\hat{\mu}_i$
\ENDFOR
\end{algorithmic}
\end{algorithm}

Algorithm~\ref{alg:CUCB} integrates our framework with CUCB. To handle non-stationary rewards from ongoing training, we use a \textbf{sliding window} of the $H$ most recent evaluation rounds. The UCB estimate balances exploration and exploitation.

For super arm selection, we adopt a \textbf{co-occurrence weighted greedy strategy}~\citep{kabbur2013fism}. We maintain a matrix $W \in \mathbb{R}^{N \times N}$ where $W_{ij}$ records reward-weighted co-occurrence of arms $i$ and $j$:
\begin{equation}
W_{ij} \leftarrow (1-\alpha) W_{ij} + \alpha \cdot \mathbf{1}[i, j \in x] \cdot r(x).
\end{equation}
During selection, we choose the highest UCB arm as seed, then greedily add arms maximizing:
\begin{equation}
\text{score}(a | \mathbb{S}) = \overline{\mu}_a + \lambda \cdot \frac{1}{|\mathbb{S}|} \sum_{b \in \mathbb{S}} W_{ab},
\end{equation}
where $\lambda$ balances individual quality and co-occurrence synergy. This favors arms with both high UCB and strong synergistic patterns.

% \section{Methodology}

% In our framework, the reward of a super arm is defined as the cumulative reward of all arms \( i \in \mathbb{S} \) at round \( t \), i.e., \( \sum_{i \in \mathbb{S}} X_i \). The expected reward satisfies \( r_{\mu}(\mathbb{S}) = \mathbb{E}[R_t(\mathbb{S})] \). In this setting, the bounded smoothness function simplifies to \( f(x) = Kx \), where \( K \) is the size of the super arm. This assumption ensures that small perturbations in individual arm rewards do not lead to disproportionate changes in the super arm's overall reward, thereby maintaining stability during training.

\section{Experiment}
%方法
We experimented on 5 commonly used standard tasks. As baselines, we use Trajectory Balance (TB) \citep{malkin2022trajectory}, Sub-Trajectory Balance (SUBTB) \citep{madan2023learning}, Detailed Balance (DB) \citep{jain2022biological,malkin2022trajectory}, LSGFN \citep{kim2023local}, Teacher \citep{kim2024adaptive} and QGFN \citep{lau2024qgfn}. All experiments are conducted on NVIDIA Tesla A100 80GB GPUs. Our evaluation centers on four key research questions (\textbf{RQs}):
\ding{46} \textbf{RQ1. How is the quality of candidates generated by {\modelname}? }
\ding{46} \textbf{RQ2. How diverse are the generated candidates? }
\ding{46} \textbf{RQ3. Why CMAB algorithms work for {\GFNs}? }
\ding{46} \textbf{RQ4. How robust is the proposed method across different hyperparameters?}
For RQ3, we performed further analysis by reporting the empirical mean rewards of each base arm in the Bit Sequence Generation task. We also conducted ablation studies, replacing the CMAB framework with alternative strategies such as Hard Pruning and Proportional Pruning to assess their effects in Appendix~\ref{strategies}. We introduce a baseline called \textbf{RandGFN} in all experiments, which replaces CMAB with random pruning. 

\noindent\textbf{Computational Fairness.} To ensure fair comparison, all methods are trained with the same number of rounds and receive the same number of training samples per round for gradient updates. While {\modelname}'s two-phase protocol incurs additional environment calls for CMAB arm selection (e.g., 100 evaluation samples vs.\ 400 training samples per round in our molecule experiments), these evaluation samples serve \textit{exclusively} for arm selection and do \textit{not} contribute to GFlowNet training. Consequently, all methods receive equivalent training signals, and the additional cost lies solely in reward evaluations. Wall-clock time and GPU memory consumption are reported in Appendix~\ref{time-memory}, demonstrating that the evaluation overhead remains modest relative to the performance gains. Following standard practice in GFlowNet literature~\citep{bengio2021flow,malkin2022trajectory}, we assume access to an efficient reward function or proxy model that evaluates generated candidates at negligible cost.

\begin{figure}[htbp]
    \centering
    \includegraphics[width=0.95\linewidth]{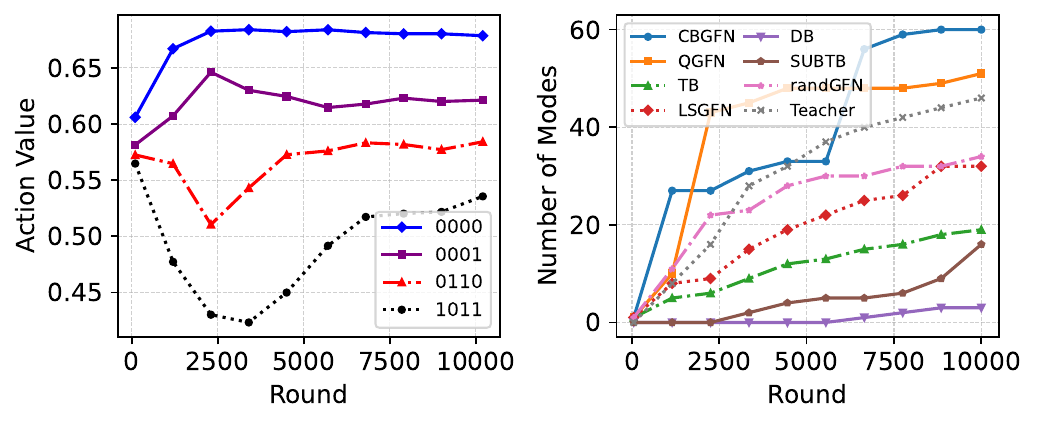}
    \caption{\textbf{Experimental results on Bit Sequence task.} Left panel shows how different action values change as the training progresses, where action values refer to the empirical mean rewards for each base arm, corresponding to $\hat{\mu_i}$ in Eq.~\ref{eq-4}. Right panel shows the mode discovered by different methods.}
    \label{fig:bit-seq-all}
\end{figure}
\subsection{Bit Sequence Generation}
\subsubsection{Task Definition}
The task is to generate binary bit sequences using the set $\{0, 1\}$ with a fixed length $n=120$ with a terminal state space of $2^{120} \approx 10^{36}$ and more intermediate states. The reward of a terminal $x$ is defined as $R(x) =  \exp(-\min_{m \in M}\mathrm{dist}(x,m))$, where $\mathrm{dist}(x,m)$ is the Levenshtein Distance of two sequences following \citep{malkin2022trajectory,zhang2023distributional}. $M$ is a predefined sequence set to be discovered as modes. The mode $m\in M$ is regarded as found if there exists a sample $x$ satisfying $\mathrm{dist}(x,m)<\delta$, where $\delta$ is a predefined parameter.

In this task, we consider a more complex version with many more intermediate states. \citet{malkin2022trajectory} considers the process as a left-to-right generation where the state space is only a simple tree. \citet{lau2024qgfn,shen2023towards} use a prepend-append MDP to induce a DAG. In our setting, instead of prepending or appending to the existing sequence, we first divide the sequence into $\lfloor \frac{n}{k} \rfloor$ positions. We can insert a generated $k$-bit into any unfilled position, resulting in a more complex DAG. 

\subsubsection{Result}
The performance comparison of different methods on the bit sequence task is illustrated in Figure \ref{fig:bit-seq-all}. {\GFNs} employing TB objective demonstrate superior results, outperforming all other objective functions. {\modelname} successfully identifies all potential modes, representing a substantial advancement in mode discovery efficiency. 
Figure \ref{fig:bit-seq-all}-a illustrates the evolution of \(\mu_i\) for various base arms (actions). As training progresses with the CUCB algorithm's action selection, the arms diverge, converging to distinct outcomes. Notably, the actions \{0000, 1111\} emerge as the highest-scoring, aligning with the construction of \(M\), where \{0000, 1111\} appear most frequently compared to other actions.

% \begin{table}[t]
%     \centering
%     \scriptsize
%     \caption{Comparison on Bit Sequence Generation. 
%     \textit{Modes} means the number of discovered modes. 
%     \textit{Top1000} denotes the average reward of the best 1000 candidates.}
%     \label{table-bit}
%     \begin{tabular}{ccc}
%         \toprule
%         Model & Modes$\uparrow$ & Top1000$\uparrow$ \\
%         \midrule
%         \textbf{{\modelname}} & \cellcolor[rgb]{.851,.851,.851}\textbf{60} 
%                         & \cellcolor[rgb]{.851,.851,.851}\textbf{3.60} \\
%         QGFN  & 51 & 3.42 \\
%         LSGFN & 32 & 2.98 \\
%         TB    & 19 & 2.89 \\
%         SUBTB & 16 & 2.84 \\
%         DB    & 3  & 2.66 \\
%         Teacher & 46 & 3.17 \\
%         {\randmodel} & 32 & 3.06 \\
%         \bottomrule
%     \end{tabular}
% \end{table}

% \begin{figure} [ht]
%    \centering
%    \subfigure[Number of modes with reward $>7.5$\label{fig:m-s-75}]{
    %       \includegraphics[width=0.48\textwidth]{figures/m_75_scaffold_add.pdf}}
    %    \subfigure[Number of modes with reward $>8$\label{fig:m-s-80}]{
        %       \includegraphics[width=0.48\textwidth]{figures/m_80_scaffold_add.pdf}}
        %    \caption{\textbf{The curve of number of modes varying with rounds ($10^3$) on Molecule Design.}}
        %    \label{fig:m-m}
        % \end{figure}

        \subsection{Molecule Design}
        \subsubsection{Task Definition}

        \begin{figure}[htbp]
            \centering
            \includegraphics[width=0.95\linewidth]{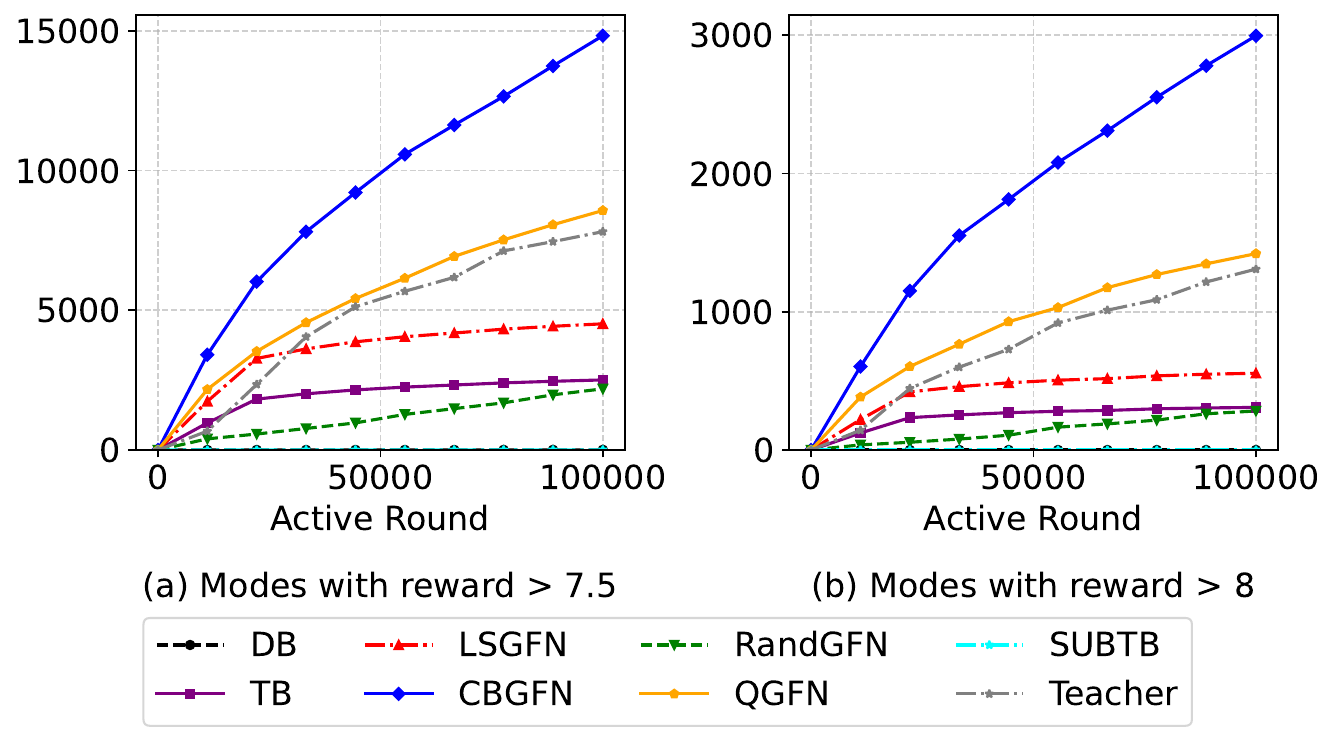}
            \caption{\textbf{The curve of number of modes varying with rounds ($10^3$) on Molecule Design.} Left panel shows the number of modes discovered with a reward $R>7.5$. Right panel shows the number of modes discovered with a reward $R>8$.}
            \label{fig:mols}
        \end{figure}
        
        % \begin{figure}[ht]
        %     \centering
        %     \includegraphics[width=\linewidth]{figures/rna1.pdf}
        %     \caption{Performance comparison on RNA-1 design tasks.}
        %     \label{fig:rna}
        % \end{figure}
        % \begin{figure}[ht]
        %     \centering
        %     \includegraphics[width=\linewidth]{figures/rna2.pdf}
        %     \caption{Performance comparison on RNA-2 design tasks.}
        %     \label{fig:rna}
        % \end{figure}
        % \begin{figure}[ht]
        %     \centering
        %     \includegraphics[width=\linewidth]{figures/rna3.pdf}
        %     \caption{Performance comparison on RNA-3 design tasks.}
        %     \label{fig:rna}
        % \end{figure}
        \begin{table}[htbp]
            \centering
            \scriptsize
            \caption{Comparison on Molecule Design. 
            \textit{Modes R$>$7.5/8} means the number of modes with a reward bigger than 7.5/8. 
            \textit{Top1000-perf} denotes the average reward of the best 1000 candidates. \textit{Top1000-similarity} denotes the tanimoto similarity of the best 1000 candidates. Best result marks exclude results with less than 1000 discovered high-scoring modes.}
            \label{table-mols}
            \begin{tabular}{ccccc}
                \toprule
                \multirow{2}{*}{Model} & \multirow{2}{*}{Modes R$>$7.5 $\uparrow$} & \multirow{2}{*}{Modes R$>$8 $\uparrow$} & \multicolumn{2}{c}{Top1000} \\
                \cmidrule(lr){4-5}
                & & & Perf$\uparrow$ & Similarity$\downarrow$ \\
                \midrule
                \rowcolor[rgb]{.851,.851,.851}
                \textbf{{\modelname}} & \textbf{14913
                } & \textbf{3207} & \textbf{8.436} & 0.49 \\
                QGFN  & 8567 & 1420 & 8.395 & 0.54\\
                LSGFN & 4514 & 555  & 8.316 & 0.53\\
                TB    & 2507 & 308  & 8.233 & 0.47\\
                SUBTB & 6    & 0    & 7.245 & 0.50\\
                DB    & 6    & 0    & 7.124 & 0.44\\
                Teacher & 7811 & 1308 & 8.364 & 0.49\\
                {\randmodel} & 2188 & 282 & 8.248 & \textbf{0.46}\\
                \bottomrule
            \end{tabular}
        \end{table}
We consider the most common scenario for GFlowNets, the fragment-based molecule generation task. The objective is to design a variety of molecules with a high reward, where the reward is given by a proxy model predicting the binding affinity to the sEH (soluble epoxide hydrolase) protein based on a docking prediction\citep{trott2010autodock}. We use the proxy model provided by \citep{bengio2021flow}.

In this task, the states are represented as molecule graphs or SMILES. The action space consists of two components: selecting which molecular stems to extend and choosing which building blocks to add. The maximum number of allowed blocks controls the size of the state space. The vocabulary of building blocks consists of 105 distinct elements, where a block has several possible attachment points(stems). We generate a molecule graph of up to 8 fragments. Therefore, the terminal state space is more than $105^8 \approx 10^{16}$.

We define each block as a base arm and choose $K$ blocks as a super arm. There are $C_{105}^K$ different base arms. We use Tanimoto similarity \citep{bender2004molecular} to distinguish different modes, with a threshold of $0.7$. Furthermore, we conduct a comprehensive analysis to examine how different values of K affect the algorithm's performance, particularly in terms of controlling its greediness.

\subsubsection{Result}
The comparative performance of various models is summarized in Table \ref{table-mols} and Figure \ref{fig:mols}. We observe
that the baseline SUBTB consistently underperforms in both mode discovery and the generation of
high-reward candidates. This pattern is not unique to our experiments; a
similar trend was reported in the molecule design task by \citet{lau2024qgfn}, suggesting that SUBTB
may face intrinsic challenges in effectively exploring complex reward landscapes. Our proposed {\modelname} demonstrates remarkable improvements in high-scoring mode discovery. During our experiments, the model successfully identified over 10,000 high-scoring modes $(R>7.5)$ with 100,000 training rounds. Furthermore, {\modelname} achieves superior performance in terms of average reward for the top 1000 candidates, outperforming all baseline methods.

Table~\ref{table-mols} reports the Tanimoto similarity of the Top-1000 candidates. While sampling within a fixed pruned subspace can increase similarity by concentrating probability mass on a narrow region, our CMAB controller repeatedly re-selects the super arm (Alg.~\ref{alg:CUCB}, line 11), effectively mixing samples across different subspaces. As a result, {\modelname} achieves the best Top-1000 reward while keeping similarity competitive (0.49), substantially lower than other pruning-based baselines such as QGFN (0.54), suggesting that the gain in quality does not come from collapsing to near-duplicate candidates.

\begin{figure}[htbp]
    \centering
    \includegraphics[width=0.92\linewidth]{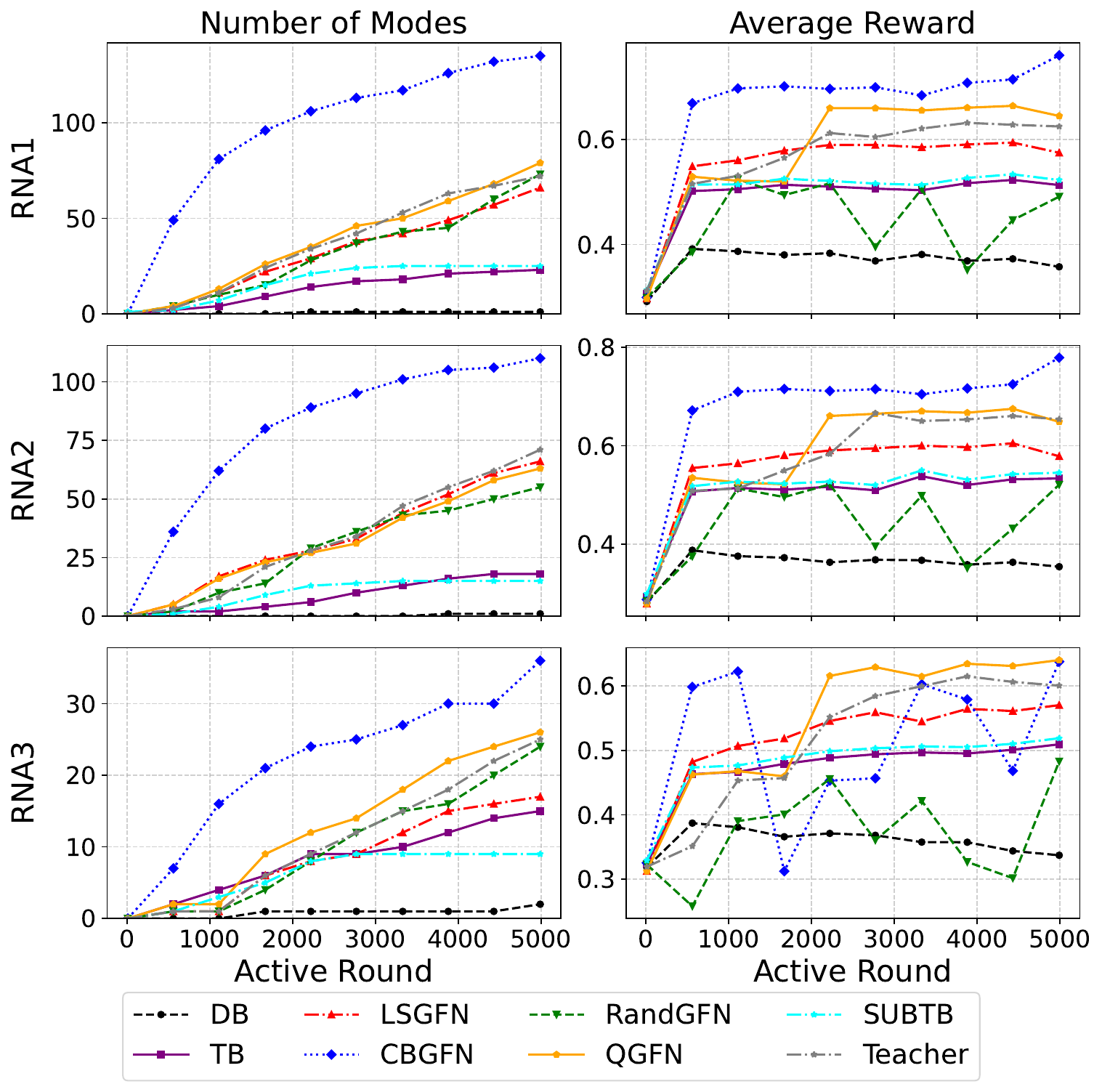}
    \caption{\textbf{Performance comparison on RNA design tasks.} Rows correspond to RNA-1, RNA-2, and RNA-3, respectively.}
    \label{fig:rna}
\end{figure}
\subsection{RNA-Binding}
\begin{figure*}[htbp]
    \centering
    \includegraphics[width=\linewidth]{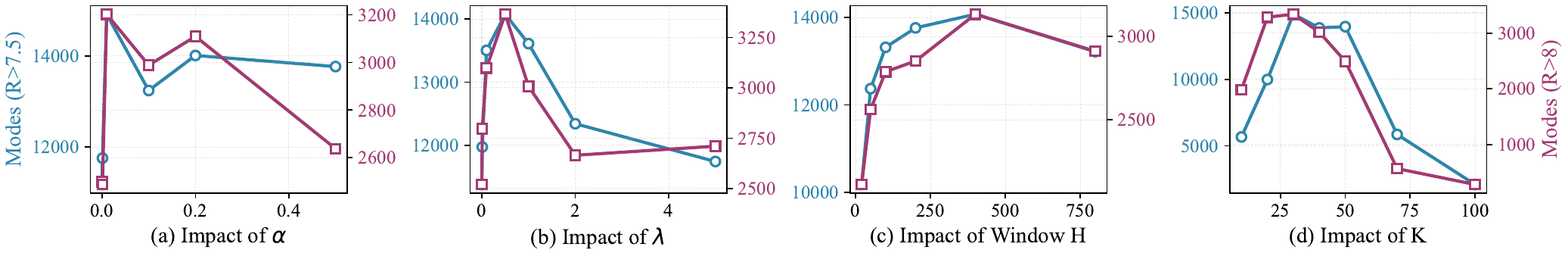}
    \caption{\textbf{Ablation study on key hyperparameters.} We evaluate sensitivity to super arm size $K$, co-occurrence update rate $\alpha$, co-occurrence weighting coefficient $\lambda$, and sliding-window size $H$ on the molecule generation task. Each panel shows the number of discovered modes ($R>7.5$) and discovered modes ($R>8$).}
    \label{fig:ablation_combined}
\end{figure*}
\subsubsection{Task Definition}
The task is to generate a string of 14 nucleobases. We use a prepend-append MDP to keep adding tokens to a string until it reaches the maximum length, following \citet{kim2023local}. There are 4 tokens: adenine (A), cytosine (C), guanine (G), and uracil (U). We conducted experiments on three different target transcriptions: L14-RNA1, L14-RNA2, and L14-RNA3 proposed by \citet{sinai2020adalead}. We treat each token as a base arm, and $K$ is set to $2/3$, denoting that we can either choose $2$ base arms or $3$ base arms as a super arm.

\subsubsection{Result}
Figure~\ref{fig:rna} reports the results on three RNA tasks, each evaluated by Number of Modes Discovered and Average Reward. {\modelname} consistently outperforms baselines: it discovers nearly twice as many modes as the strongest competitor and achieves higher average and Top-1000 rewards in Tasks 1 and 2, while remaining competitive in Task 3. The slight instability in Task 3 stems from averaging over only 10 rounds, during which the agent may explore subspaces with suboptimal rewards or insufficiently learned dynamics.
We further analyze the impact of the reward exponent $\beta$ (inverse temperature) on RNA tasks in Section~\ref{sec:beta_rna}.

\subsection{Ablation Study and Sensitivity Analysis}
\label{sec:hyper_sensitivity}
Our CMAB-guided sampling introduces several key hyperparameters: the super arm size $K$, the co-occurrence update rate $\alpha$ (Eq.~10), the co-occurrence weighting coefficient $\lambda$ (Eq.~11), and the sliding-window size $H$ (Algorithm~\ref{alg:CUCB}). To assess their impact, we conduct systematic ablation studies on the molecule generation task.

Figure~\ref{fig:ablation_combined} demonstrates that {\modelname} exhibits robust performance across reasonable hyperparameter ranges. \textbf{Super arm size $K$} is a fundamental parameter determining how many base arms are retained in each subspace. When $K$ is too small (e.g., $K=10$), the search space becomes overly constrained, limiting diversity and reducing the number of discovered high-reward modes. Conversely, when $K$ is too large (e.g., $K=100$), the subspace pruning becomes less effective, and the framework approaches the behavior of vanilla GFlowNets. Moderate values ($K \in [20, 50]$) strike an optimal balance, keeping the subspace sufficiently compact to guide exploration toward high-reward regions while preserving enough diversity to discover multiple distinct modes. \textbf{Co-occurrence update rate $\alpha$} controls the smoothness of exponential moving averaging in $W$: moderate values ($\alpha \in [0.01, 0.1]$) maintain stable performance, while extreme values (0 or 0.2) show marginal degradation. \textbf{Co-occurrence weighting $\lambda$} interpolates between individual arm quality (UCB) and pairwise synergy ($W$): setting $\lambda=0$ slightly underperforms, while moderate values (0.25--1.0) consistently improve performance by exploiting synergistic combinations. Excessively large $\lambda$ (2.0) degrades results by over-emphasizing spurious correlations. \textbf{Sliding-window size $H$} balances bias and variance: very small windows ($H=50$) suffer from high-variance estimates, while moderate windows ($H = 400$) maintain stable performance. 

Overall, the method exhibits smooth performance surfaces with no brittle failure modes. Hyperparameter tuning primarily affects the speed and stability of bandit adaptation rather than qualitative behavior. For new tasks, we recommend $K$ approximately 40--60\% of the total number of base arms $N$, $\alpha \in [0.01, 0.1]$, $\lambda \in [0.25, 1.0]$, and $H$ covering a few hundred evaluation rounds as a robust starting point.

\begin{table}[h]
\centering
\scriptsize
\caption{Comparison of CMAB methods on molecule generation task. Results shown as percentages relative to CUCB baseline (100\%).}
\label{tab:cmab_comparison}
\begin{tabular}{lccc}
\toprule
Method & Modes $R>7.5$ & Modes $R>8$ & Top-K Reward \\
\midrule
CUCB (baseline) & 100.0\% & 100.0\% & 100.0\% \\
CTS  & 93.5\% & 85.0\% & 99.9\% \\
ESCB & 103.8\% & 103.4\% & 100.1\% \\
\bottomrule
\end{tabular}
\end{table}

We conduct an additional experiment where we replaced CUCB with other CMAB algorithms. Table~\ref{tab:cmab_comparison} reports the results as percentages relative to CUCB (baseline = 100\%). All three methods achieve competitive performance, with ESCB showing slight improvements in mode discovery while CTS underperforms moderately.

\section{Conclusion and Limitation}
\textbf{Conclusion:} This paper proposes {\modelname}, a method designed to enhance the greediness of the sampling process while preserving the diversity of generated candidates. We begin by partitioning the entire state space into multiple subspaces. Next, we employ the CUCB algorithm to effectively balance exploration and exploitation and find the optimal subspaces. To address the challenges posed by narrow-deep network architectures, we propose techniques to transform them into more balanced wide-deep structures. Experimental results across various tasks demonstrate the effectiveness and efficiency of {\modelname}.

\textbf{Limitations:} 
% While the proposed method demonstrates promising results, it incurs additional computational overhead compared to vanilla GFlowNets due to the required sampling process. 
Although our framework is theoretically applicable to tasks like listwise recommendation and combinatorial optimization problems, empirical validation on these tasks remains for future work. 
% Furthermore, while we address non-trivial challenges in our study, a comprehensive analysis of parameter scaling and computational efficiency across the benchmarked methods has not been explored. 
Another limitation of the proposed method is that it assumes a fixed reward distribution in the environment. In scenarios where the high-reward state space shifts during training, the benefits are limited and may even disappear. 
\label{limitation}
\nocite{langley00}

\bibliography{example_paper}
\bibliographystyle{icml2026}

%%%%%%%%%%%%%%%%%%%%%%%%%%%%%%%%%%%%%%%%%%%%%%%%%%%%%%%%%%%%%%%%%%%%%%%%%%%%%%%
%%%%%%%%%%%%%%%%%%%%%%%%%%%%%%%%%%%%%%%%%%%%%%%%%%%%%%%%%%%%%%%%%%%%%%%%%%%%%%%
% APPENDIX
%%%%%%%%%%%%%%%%%%%%%%%%%%%%%%%%%%%%%%%%%%%%%%%%%%%%%%%%%%%%%%%%%%%%%%%%%%%%%%%
%%%%%%%%%%%%%%%%%%%%%%%%%%%%%%%%%%%%%%%%%%%%%%%%%%%%%%%%%%%%%%%%%%%%%%%%%%%%%%%
\newpage
\appendix
\onecolumn

\section{LLM Usage Statement}
We only use LLMs as a language optimization tool to polish sentences, improving their readability and fluency. The LLM did not contribute to the scientific ideas, algorithm design, or experimental setup. All substantive content, reasoning, and conclusions are entirely the product of the authors.   We accept full responsibility for all content in the paper, including parts refined or corrected by the LLM, and affirm that no text generated by the LLM constitutes original scientific contributions attributed to it. 

\section{Theorem Proof}

We first state the precondition required for our theoretical analysis.

\textbf{Precondition. Exploration Guarantee}
\label{assump:exploration}
The forward policy $\pi_t$ maintains a minimum exploration probability $\epsilon > 0$ throughout training, i.e., for any valid transition $s \to s'$, we have $\pi_t(s'|s) \geq \epsilon / |\mathcal{A}_s|$, where $|\mathcal{A}_s|$ is the number of available actions at state $s$.

Then we establish the convergence of GFlowNet policies and subsequently the convergence of base arm reward distributions.

\begin{remark}[Scope of the theoretical claims]
The trajectory balance (TB) objective is \emph{consistent}: any exact solution of the TB equations induces the reward-proportional terminal distribution. However, for general function approximation (e.g., neural networks) and non-convex optimization, a global convergence guarantee of TB training is not established in full generality. Our analysis below therefore separates (i) TB consistency (a property of the optimum) from (ii) policy stabilization during training (an assumption that can hold when the TB loss plateaus and updates become small).
\end{remark}

\begin{lemma}[TB Consistency (Zero-loss Implies Target Distribution)]
\label{lemma:gfn_convergence}
Assume the environment induces a finite DAG with initial state $s_0$ and terminal set $\mathcal{X}$, and let $P_F(\cdot|\cdot;\theta)$ and $P_B(\cdot|\cdot;\theta)$ be valid forward/backward transition kernels with full support. If the TB constraints hold for all trajectories $\tau=(s_0\rightarrow\cdots\rightarrow x)$, i.e.,
\begin{equation}
Z_\theta \prod_t P_F(s_t|s_{t-1};\theta)=R(x)\prod_t P_B(s_{t-1}|s_t;\theta),
\label{eq:tb_identity}
\end{equation}
equivalently $\mathcal{L}_{\text{TB}}(\theta)=0$, then the induced terminal distribution satisfies $\pi_\theta(x)=\frac{R(x)}{\sum_{x'\in\mathcal{X}}R(x')}$.
\begin{equation}
\pi_\theta(x) = \frac{R(x)}{\sum_{x' \in \mathcal{X}} R(x')}.
\end{equation}
\end{lemma}

\begin{proof}
Fix a terminal state $x\in\mathcal{X}$ and sum~\eqref{eq:tb_identity} over all forward trajectories ending at $x$. The left-hand side becomes
\[
Z_\theta \sum_{\tau:s_0\rightarrow x}\prod_t P_F(s_t|s_{t-1};\theta)=Z_\theta \,\pi_\theta(x),
\]
where $\pi_\theta(x)$ is the probability of sampling terminal $x$ under the forward policy. The right-hand side becomes
\[
R(x)\sum_{\tau:s_0\rightarrow x}\prod_t P_B(s_{t-1}|s_t;\theta)=R(x),
\]
since the backward kernel defines a valid stochastic process on the reversed DAG and the total probability of reaching $s_0$ from $x$ by following $P_B$ is $1$ (i.e., the sum over all backward trajectories from $x$ to $s_0$ equals $1$). Hence $Z_\theta\,\pi_\theta(x)=R(x)$ for all $x$, and normalizing over $\mathcal{X}$ yields $\pi_\theta(x)\propto R(x)$. This is the standard consistency property of TB~\citep{bengio2021flow,malkin2022trajectory}.
\end{proof}

\begin{lemma}[Convergence of Sampling Distribution]
\label{lemma:sampling_convergence}
Assume the terminal distributions $\{\pi_t\}$ produced by training satisfy pointwise convergence $\pi_t(x)\to \pi_\infty(x)$ for all $x\in\mathcal{X}$, and that $\sum_{x\in\mathcal{X}_i}\pi_\infty(x)>0$. For any base arm $i$, define the conditional sampling distribution $p_t(x|i)=\pi_t(x)/\sum_{x'\in\mathcal{X}_i}\pi_t(x')$. Then:
\begin{equation}
\lim_{t \to \infty} \sum_{x \in \mathcal{X}_i} |p_t(x|i) - p_\infty(x|i)| = 0,
\end{equation}
where $p_\infty(x|i)=\pi_\infty(x)/\sum_{x'\in\mathcal{X}_i}\pi_\infty(x')$. Moreover, if $\pi_\infty$ satisfies the TB constraints (e.g., $\mathcal{L}_{\text{TB}}(\theta_t)\to 0$ so that $\pi_\infty=\pi^*$ as in Lemma~\ref{lemma:gfn_convergence}), then $p_\infty(\cdot|i)=p^*(\cdot|i)$.
\end{lemma}

\begin{proof}
Since $\mathcal{X}_i$ is finite and $\pi_t(x)\to \pi_\infty(x)$ for each $x\in\mathcal{X}$, we have:
\begin{align}
p_t(x|i) &= \frac{\pi_t(x)}{\sum_{x' \in \mathcal{X}_i} \pi_t(x')} \xrightarrow{t \to \infty} \frac{\pi_\infty(x)}{\sum_{x' \in \mathcal{X}_i} \pi_\infty(x')} = p_\infty(x|i).
\end{align}
The convergence in total variation follows from the finiteness of $\mathcal{X}_i$ and pointwise convergence.
\end{proof}

\textbf{Theorem 1.}
Assume (i) $\pi_t$ converges pointwise to a limiting terminal distribution $\pi_\infty$ (policy stabilization), and (ii) candidates in $C_{i,t}$ are sampled i.i.d.\ from $p_t(\cdot|i)$ at each round. Then the reward distribution of each base arm $i$ converges to a stable distribution. Specifically, let
\begin{equation}
X_{i,t} = \frac{1}{|C_{i,t}|}\sum_{x \in C_{i,t}} r(x),
\end{equation}
where $C_{i,t}$ is the set of $n$ candidates sampled at round $t$ that contain base arm $i$, and $r(x) = R(x)/R_{\max}$ is the normalized reward. Then $X_{i,t}$ converges in distribution to a random variable $X_i^*$ as $t \to \infty$.

\textbf{Proof.}
\label{Proof 1}
We establish the convergence through the following steps.

\noindent\textbf{Step 1: Characterization of the Limiting Distribution.}\\
At round $t$, candidates in $C_{i,t}$ are sampled i.i.d.\ from the distribution $p_t(\cdot|i)$ over $\mathcal{X}_i$. The empirical mean reward is:
\begin{equation}
X_{i,t} = \frac{1}{|C_{i,t}|}\sum_{x \in C_{i,t}} r(x).
\end{equation}
Define the limiting expected reward under the stabilized (limiting) conditional distribution:
\begin{equation}
\mu_{i,\infty} = \sum_{x \in \mathcal{X}_i} p_\infty(x|i) \cdot r(x).
\end{equation}
If additionally the TB constraints are solved asymptotically (e.g., $\mathcal{L}_{\text{TB}}(\theta_t)\to 0$ so that $\pi_\infty=\pi^*$), then $p_\infty(\cdot|i)=p^*(\cdot|i)$ and $\mu_{i,\infty}=\mu_i^*$.

\noindent\textbf{Step 2: Convergence of Conditional Expectations.}\\
Let $\mu_{i,t} = \mathbb{E}[r(X) | X \sim p_t(\cdot|i)] = \sum_{x \in \mathcal{X}_i} p_t(x|i) \cdot r(x)$ be the expected reward under the current sampling distribution. By Lemma~\ref{lemma:sampling_convergence}, $p_t(x|i) \to p_\infty(x|i)$ for all $x \in \mathcal{X}_i$. Since $r(x) \in [0,1]$ is bounded, the dominated convergence theorem yields:
\begin{equation}
\mu_{i,t} = \sum_{x \in \mathcal{X}_i} p_t(x|i) \cdot r(x) \xrightarrow{t \to \infty} \sum_{x \in \mathcal{X}_i} p_\infty(x|i) \cdot r(x) = \mu_{i,\infty}.
\end{equation}

\noindent\textbf{Step 3: Convergence in Distribution.}\\
For fixed $t$, given $n_t = |C_{i,t}|$ samples, the empirical mean $X_{i,t}$ has:
\begin{equation}
\mathbb{E}[X_{i,t}] = \mu_{i,t}, \quad \text{Var}(X_{i,t}) = \frac{\sigma_{i,t}^2}{n_t},
\end{equation}
where $\sigma_{i,t}^2 = \sum_{x \in \mathcal{X}_i} p_t(x|i)(r(x) - \mu_{i,t})^2$ is the variance under $p_t(\cdot|i)$.

We now show convergence in distribution. For any $\delta > 0$, by Chebyshev's inequality:
\begin{equation}
\mathbb{P}(|X_{i,t} - \mu_{i,t}| > \delta) \leq \frac{\sigma_{i,t}^2}{n_t \delta^2} \leq \frac{1}{4n_t \delta^2},
\end{equation}
where we used $\sigma_{i,t}^2 \leq 1/4$ since $r(x) \in [0,1]$.

Let $\varepsilon > 0$ be arbitrary. Choose $T$ large enough such that for all $t > T$: (i) $|\mu_{i,t} - \mu_{i,\infty}| < \varepsilon/2$ (by Step 2), and (ii) $n_t$ is sufficiently large that $1/(4n_t\delta^2) < \varepsilon/2$ for $\delta = \varepsilon/2$. Then for $t > T$:
\begin{align}
\mathbb{P}(|X_{i,t} - \mu_{i,\infty}| > \varepsilon) &\leq \mathbb{P}(|X_{i,t} - \mu_{i,t}| > \varepsilon/2) + \mathbb{P}(|\mu_{i,t} - \mu_{i,\infty}| > \varepsilon/2) \\
&< \varepsilon/2 + 0 = \varepsilon/2.
\end{align}

This shows $X_{i,t} \xrightarrow{p} \mu_{i,\infty}$. More generally, the distribution of $X_{i,t}$ converges to a distribution concentrated around $\mu_{i,\infty}$. In the large-sample regime where $n_t \to \infty$, by the Central Limit Theorem, $(X_{i,t} - \mu_{i,t})\sqrt{n_t} \xrightarrow{d} \mathcal{N}(0, \sigma_{i,\infty}^{2})$, where $\sigma_{i,\infty}^{2} = \sum_{x \in \mathcal{X}_i} p_\infty(x|i)(r(x) - \mu_{i,\infty})^2$.

Combining the convergence of $\mu_{i,t} \to \mu_{i,\infty}$ and $\sigma_{i,t}^2 \to \sigma_{i,\infty}^{2}$ with Slutsky's theorem, we conclude that $X_{i,t}$ converges in distribution to $X_i^*$ concentrated around $\mu_{i,\infty}$ (e.g., approximately $\mathcal{N}(\mu_{i,\infty}, \sigma_{i,\infty}^{2}/n^*)$ for large fixed batch size $n^*$), or concentrates at $\mu_{i,\infty}$ as $n_t \to \infty$.

\noindent\textbf{Step 4: Uniform Integrability (Optional Strengthening).}\\
Since $r(x) \in [0,1]$, we have $X_{i,t} \in [0,1]$ uniformly, implying $\{X_{i,t}\}$ is uniformly integrable. Combined with convergence in probability, this yields $\mathbb{E}[X_{i,t}] \to \mu_{i,\infty}$, establishing convergence in $L^1$ as well.

This completes the proof of Theorem 1.

\textbf{Theorem 2 (Asymptotic Approximate Distributional Invariance).}
\label{Proof 2}
Under the two-phase protocol (restricted training on $\mathbb{S}$, unrestricted evaluation on $ALL$), assume there exists a sequence $\delta_t\downarrow 0$ such that for any base arm $a$ and any two super arms $\mathbb{S},\mathbb{S}'$ that could be used in Phase~1 at round $t$, the resulting evaluation-time conditional sampling distributions satisfy
\begin{equation}
\big\|p_t^{\mathbb{S}}(\cdot|a)-p_t^{\mathbb{S}'}(\cdot|a)\big\|_1 \le \delta_t.
\end{equation}
Then the distribution of the unrestricted evaluation reward estimate $X_a^t$ becomes asymptotically invariant to the Phase~1 super arm, i.e., the CMAB ``distributional consistency'' assumption holds approximately in late training.

\textbf{Proof.}
In Phase~2, $X_a^t$ is an empirical average of bounded rewards (after normalization, $r(\cdot)\in[0,1]$) over samples drawn from $p_t^{\mathbb{S}}(\cdot|a)$. For any bounded function $f$ with $\|f\|_\infty\le 1$, the difference in expectations is bounded by total variation:
\[
\Big|\mathbb{E}_{p_t^{\mathbb{S}}}[f]-\mathbb{E}_{p_t^{\mathbb{S}'}}[f]\Big|\le \tfrac{1}{2}\big\|p_t^{\mathbb{S}}(\cdot|a)-p_t^{\mathbb{S}'}(\cdot|a)\big\|_1 \le \tfrac{1}{2}\delta_t.
\]
Taking $f(x)=r(x)$ yields $|\mathbb{E}[X_a^t|\mathbb{S}]-\mathbb{E}[X_a^t|\mathbb{S}']|\le \tfrac{1}{2}\delta_t$. Since $\delta_t\to 0$, the dependence of the evaluation-time arm statistics on the preceding training subspace vanishes asymptotically.

\section{Regret Analysis}
% \subsection{Key Assumptions of CMAB in Flow Networks}
% There are two assumptions required by CMAB methods for regret analysis. 

% \textbf{Monotonicity.} For a super arm $\mathbb{S}$, the expected reward $r_\mu(\mathbb{S})$ is non-decreasing in the reward vector $\mu$. This assumption is natural in flow networks because if all base arms (actions) within a super arm \( \mathbb{S} \) exhibit higher expected rewards (i.e., \( \mu'_i \geq \mu_i \) for all \( i \in \mathbb{S} \)), it implies that the flow network under \( \mu' \) is better optimized than under \( \mu \).

% \textbf{Bounded Smoothness.} There exists an increasing function $f$ such that
% \begin{equation}
%     |r_{\mu}(\mathbb{S}) - r_{\mu'}(\mathbb{S})| < f\left(\max_{i \in \mathbb{S}} |\mu_i - \mu'_i|\right).
% \end{equation}
% This ensures stability: small perturbations in individual arm rewards do not cause disproportionate fluctuations at the super-arm level.  

% \subsection{Absence of Optimal Super Arms}
% The CMAB framework typically relies on an oracle capable of providing an \( (\alpha, \beta) \)-approximation of the optimal super arm (i.e., a subset of base arms that maximizes the expected reward) to make a tight analysis of the regret bound. However, in flow networks, identifying the exact optimal super arm is computationally prohibitive due to the combinatorial explosion of possible states. Even an \( (\alpha, \beta) \)-approximation of the optimal super arm is impractical because the $\alpha,\beta$ might change when the flow network is different.
\subsection{Regret Bound}
The flow network evolves dynamically throughout training, resulting in time-varying reward distributions for each base arm $i$. While these distributions can stabilize in late training under policy-stabilization assumptions, the inherent non-stationarity introduces considerable uncertainty into the system. Such temporal variability presents significant challenges for deriving a precise regret bound for the learning algorithm.
% Following the construction of \cite{chen2013combinatorial}, we define a super arm $S$ as $bad$ if $r_{\mu}^t(S)< opt_{\mu}^t$ at round t. We define $S_B^t = \{S\ |\ r_{\mu}^t(S)< opt_{\mu}^t\}$
% as the set of bad super arms at round t.

% While the true regret bound cannot be quantified without knowledge of the optimal arm, we can construct empirical regret metrics to evaluate algorithm performance. 
However, we can still construct empirical regret metrics to evaluate algorithm performance. A practical approach is to use the best empirical arm observed up to time $t$, denoted as $ \hat{\mu}_t(\mathbb{S}^*)$, as a proxy for the unknown optimal mean reward. The estimated cumulative regret is then computed as:
\begin{equation}
    \hat{R}(T)=\sum_{t=1}^T ( \hat{\mu}_t(\mathbb{S}^*)-\hat{\mu_t}(\mathbb{S})).
\end{equation}
We employ the metric of empirical cumulative regret to quantify the performance gap between consistently selecting the currently known optimal super arm and the super arm chosen by our algorithm. To comprehensively evaluate our method's efficacy in minimizing regret, we introduce a baseline random strategy called {\randmodel} that uniformly selects $K$ base arms at each round.

Figure \ref{fig:regret} presents the cumulative regret curves for {\modelname} and the {\randmodel}, revealing distinct performance differences across tasks:  

1. Bit Sequence Generation Task (left panel): The random policy exhibits competitive performance, resulting in a moderate regret gap between the two models. This suggests that simple heuristics suffice in this simpler task setting.  

2. Molecule Design Task (middle panel): {\modelname} demonstrates substantial improvement, achieving a significantly lower cumulative regret than {\randmodel}. This highlights the proposed model's effectiveness in optimizing structured, complex objectives.  

3. L14-RNA1 Task (right panel): The regret gap widens again, where {\modelname} substantially outperforms {\randmodel}, indicating its ability to handle intricate combinatorial challenges in nucleic acid sequence optimization.  

\begin{figure}[ht] {
   \centering
    \includegraphics[width=\textwidth]{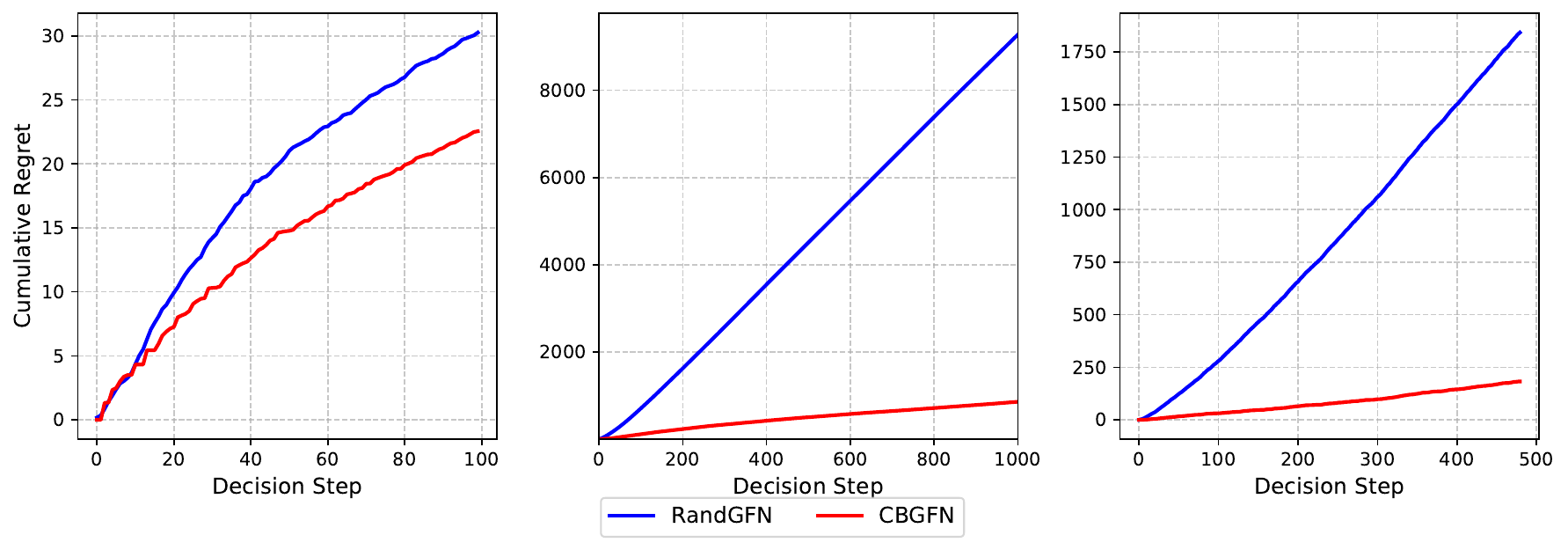}}
   \caption{\textbf{Experimental result on cumulative regret with different tasks.} Left: Cumulative regret of models in Bit Sequence Generation Task. Center: Cumulative regret of models in Molecule Design Task. Right: Cumulative regret of models in L14-RNA1 Task.}
   \label{fig:regret}
\end{figure}
%%%%%%%%%%%%%%%%%%%%%%%%%%%%%%%%%%%%%%%%%%%%%%%%%%%%%%%%%%%%%%%%%%%%%%%%%%%%%%%
%%%%%%%%%%%%%%%%%%%%%%%%%%%%%%%%%%%%%%%%%%%%%%%%%%%%%%%%%%%%%%%%%%%%%%%%%%%%%%%
\section{Experiment details: Bit Sequence}
\begin{table}[t] 
    \centering
    \caption{Key hyperparameter setting in Bit Sequence Generation task}
    \begin{tabular}{lll}
    \toprule
         \textbf{Parameter} & \textbf{Value}  \\ \midrule
         Batch size & 16  \\
         Number of steps & 10000  \\
         k-bits & 4 \\
         Lamda & 1.9 \\
         Learning rate & 1e-3\\
         Z Learning rate & 1e-3\\
         $\beta$ & 2\\
         Explore Epsilon & 0.01 \\
         K & 4 \\
         Decision Interval & 100 \\
         \bottomrule
    \end{tabular}
        \label{bit-hyper}
    \label{tab:hyper-bit}
\end{table}

Here we present the hyperparameter configuration for our bit sequence generation experiments (Table \ref{tab:hyper-bit}). While adopting the baseline framework from \citet{malkin2022trajectory}, we reduce the training steps from 50,000 to 10,000. Each action is represented by $k=4$ bits, and through empirical validation, we selected K=4 candidate arms from $\{2,4,6,8,10\}$. For the CMAB algorithm, we determined 100 steps to be the optimal decision interval after evaluating candidates from $\{50,100,200,300,400,500\}$. 

\subsection{Optimal Super Arm Configuration}
The set $M$ is constructed through random combinations of substrings derived from the base patterns $\{00000000, 11111111, 11110000, 00001111, 00111100\}$. For the case where $K=4$, we can analytically determine the optimal super arm configuration as $\mathbb{S}=\{0000,1111,1100,0011\}$. This configuration enables perfect mode identification within its substate space, achieving an average distance of $0.0$ to all target modes. We evaluate the performance when consistently selecting this optimal super arm configuration in Figure~\ref{fig:bit-op}.

% \label{optimal}
% \begin{wrapfigure}{r}{0.5\textwidth}
%    \includegraphics[width=0.5\textwidth]{figures/episilon-modes.pdf}
%    \caption{\textbf{{Experimental results on Bit Sequence Generation with different $\epsilon$ values.}} The plot shows how exploration rate affects mode discovery and sample quality in our framework.}
%    \label{fig:epsilon}
%    \vspace{-35pt}
% \end{wrapfigure}
% \subsection{Different exploration parameter $\epsilon$}

\begin{figure}[t]
    \centering
    \begin{subfigure}{0.48\linewidth}
        \centering
        \includegraphics[width=\linewidth]{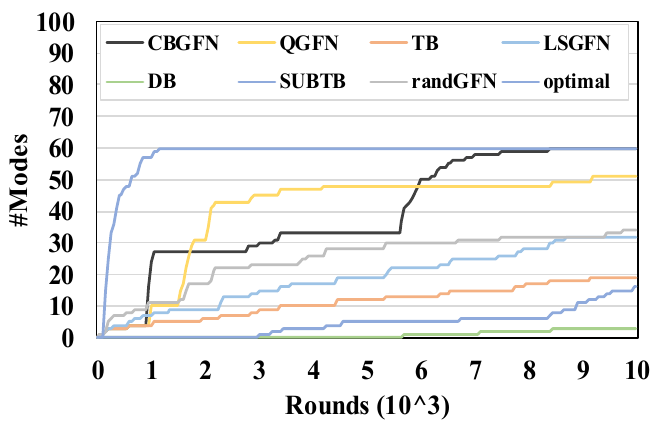}
        \caption{Average reward of each individual arm}
        \label{fig:bit-op}
    \end{subfigure}
    \hfill
    \begin{subfigure}{0.48\linewidth}
        \centering
        \includegraphics[width=\linewidth]{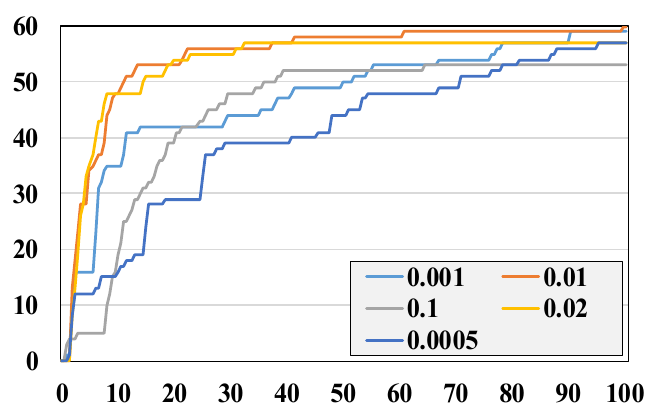}
        \caption{Number of modes discovered}
        \label{fig:bit-epsilon}
    \end{subfigure}
    \caption{Supplementary Experimental Results for the Bit Sequence Task.}
    
\end{figure}

\subsection{Experiments on different $\epsilon$}
We conduct a systematic investigation of the exploration parameter $\epsilon$, which controls the probability of random actions in the {\GFNs} framework. This parameter critically influences the sampling behavior of the flow network within substate spaces. As shown in Figure~\ref{fig:bit-epsilon}, through comprehensive testing across $\epsilon$ values $\{0.1, 0.01, 0.02, 0.001, 0.0005\}$, we find that $\epsilon = 0.01$ demonstrates superior performance in terms of mode discovery and sample quality.

The results indicate that moderate exploration ($\epsilon = 0.01$) achieves the best balance between exploration and exploitation, while higher values lead to excessive randomness and lower values result in insufficient exploration of the state space.

\section{Experiment details: Molecule Design}
\label{F}
We present the hyperparameter configurations for our Molecule Design Task experiments, as detailed in Table \ref{tab:hyper-mols}. Building upon the framework established by \citet{bengio2021flow}, we maintain their default parameter settings while introducing specific optimizations. After empirical evaluation, we set the number of base arms $K$ to 30, selected from $\{8,10,20,30,40,50,60,70,80,90,100\}$, and determined the optimal decision interval for the CMAB algorithm to be $400$ from $\{50,100,200,300,400,500\}$.
\begin{table*}[t] 
\caption{Key hyperparameter setting in Molecule Design}
    \centering
    \label{tab:hyper-mols}
    \begin{tabular}{lll}
    \toprule
         \textbf{Parameter} & \textbf{Value}  \\ \midrule
         Batch size & 4  \\
         Number of steps & 100000  \\
         Lamda & 0.99 \\
         Learning rate & 5e-4\\
         $Z$ Learning rate & 5e-3\\
         Tanimoto Similarity Threshold & 0.7\\
         $\beta$ & 8\\
         Explore Epsilon & 0.05 \\
         $K$ & 30 \\
         Decision Interval & 400 \\
         \bottomrule
    \end{tabular}
    \label{mols-hyper}
\end{table*}

% \subsection{Experiments on different $K$}
% \label{diff-k}
% The hyperparameter $K$, which determines the size of the super arm, serves as our primary mechanism for controlling the method's greediness. In our molecule generation experiments, we evaluate 10 distinct $K$ values ranging from 8 to 100. Smaller $K$ values correspond to greedier selections, as we restrict our choice to only the top $K$ base arms, resulting in more compact sub-state spaces. However, this increased greediness comes at the cost of reduced candidate diversity, as demonstrated in Figure \ref{fig:diff-k}. While $K \in \{8,10,20\}$ yields higher average scores among the top 1000 candidates, it significantly compromises molecular diversity (measured by the number of distinct modes). Through a comprehensive evaluation, we identify $K=30$ as the optimal setting for our molecular design task, achieving an effective balance between candidate quality and diversity.

% \begin{figure} [ht]
%    \centering
%     \includegraphics[width=\textwidth]{figures/diff-k.pdf}
%    \caption{\textbf{Experimental result on Molecule Design with different $k$ with 100,000 rounds.} Left: The number of modes R>7.5 with a Tanimoto similarity threshold of 0.7. Center: The number of modes R>8 with a Tanimoto similarity threshold of 0.7. Right: The average reward of the top 1000 high-scoring samples.}
%    \label{fig:diff-k}
% \end{figure}

\section{Experiment details: RNA-Binding}
\label{G}
In this section, we give the hyperparameters used for each of our experiments' RNA-Binding Task as shown in Table~\ref{tab:hyper-rna}.
In our experimental setup, the learning rate of $1 \times 10^{-4}$ is selected from  $ \{1 \times 10^{-5},1 \times 10^{-4},1 \times 10^{-3},5 \times 10^{-3} \} $ and the Z learning rate of $1 \times 10^{-2}$ is selected from $ \{1 \times 10^{-5},1 \times 10^{-4},1 \times 10^{-3},5 \times 10^{-3} \} $. The Lambda for SUBTB uses 0.9 out of $\{0.8,0.9,0.99,0.999\}$. The explore epsilon is used to control the random action probability, five values are tested, including $\{0.1,0.01,0.001,0.0001,0.0005\}$. We set the reward exponent $\beta$ to 8 from $\{3,4,5,6,7,8,9,10\}$. $K$ is the number of base arms to compose the super arm selected from $\{1,1/2,2,2/3,3,3/4,4\}$, where $x/x+1$ denotes we can choose $x$ or $x+1$ base arms as a super arm. Please note that the RNA1 environment we use is a little different from Teacher\citep{kim2024adaptive}. Our environment has 1590 modes in total, but \citet{kim2024adaptive} has 8967 modes. We replace the environment file of Teacher with our environment file to make a fair comparison. The environment is constructed following \citet{kim2023local}.
\begin{table*}[t] 
    \centering
    \caption{Key hyperparameter setting in RNA-Binding task}
    \label{tab:hyper-rna}
    \begin{tabular}{lll}
    \toprule
         \textbf{Parameter} & \textbf{Value}  \\ \midrule
         Batch size & 32  \\
         Number of steps & 5000  \\
         RNA length & 14 \\
         MDP style & Prepend and Append \\
         Lamda & 0.9 \\
         Learning rate & 1e-4\\
         Z Learning rate & 1e-2\\
         Mode metric & Hamming Ball 1 \\
         $\beta$ & 20\\
         Explore Epsilon & 0.01 \\
         K & 2/3 \\
         Decision Interval & 50 \\
         \bottomrule
    \end{tabular}
\label{rna-hyper}
\end{table*}

\subsection{Different settings: exponent $\beta$}
\label{sec:beta_rna}
\begin{figure} [ht]
   \centering
    \includegraphics[width=\textwidth]{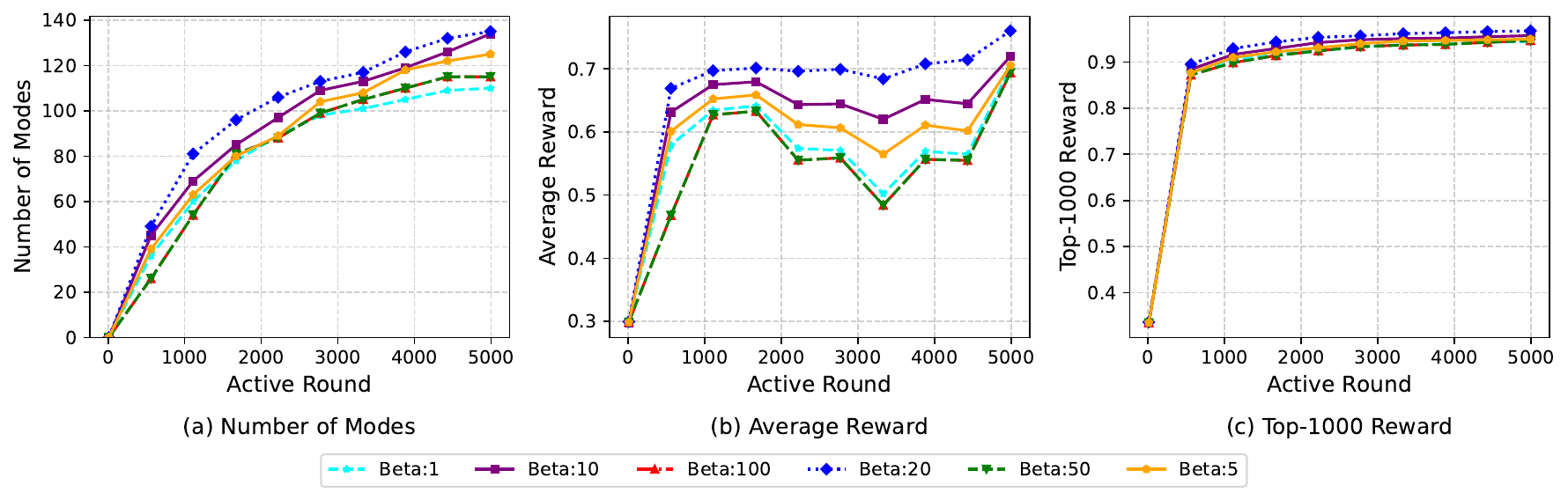}
   \caption{\textbf{Experimental result on RNA-Binding Task 1 with different $\beta$.}}
   \label{beta1}
\end{figure}

Since introduced by \citet{bengio2021flow}, there is already a useful technique to increase the greediness of {\GFNs}, that is the exponent $\beta$. The adjusted reward function is formulated as $\hat{R(x)} = R(x)^\beta$. The higher $\beta$ makes the model greedier but at the cost of greater numerical instability. Besides, since the middle-reward regions are adjusted into low-reward regions, the diversity is also reduced, leading to mode collapse \citep{lau2024qgfn}. The choice of exponent $\beta$ critically influences the behavior of the CUCB algorithm, as it directly modulates the reward scaling of individual arms. We experimented on different settings of exponent $\beta$ as shown in Figure\ref{beta1}, Figure\ref{beta2} and Figure\ref{beta3}.

In the L14-RNA1 task, models with varying $\beta$ values consistently identified over 100 distinct modes, with the $\beta = 20$ configuration demonstrating superior performance compared to other settings. Notably, the model with $\beta = 1$ exhibited significantly poorer performance relative to other $\beta$ values. Regarding average reward metrics, the $\beta = 20$ model achieved substantially better results, establishing a clear performance gap over other configurations. However, all models showed comparable performance when evaluating the top 1000 rewards.
\begin{figure} [ht]
   \centering
    \includegraphics[width=\textwidth]{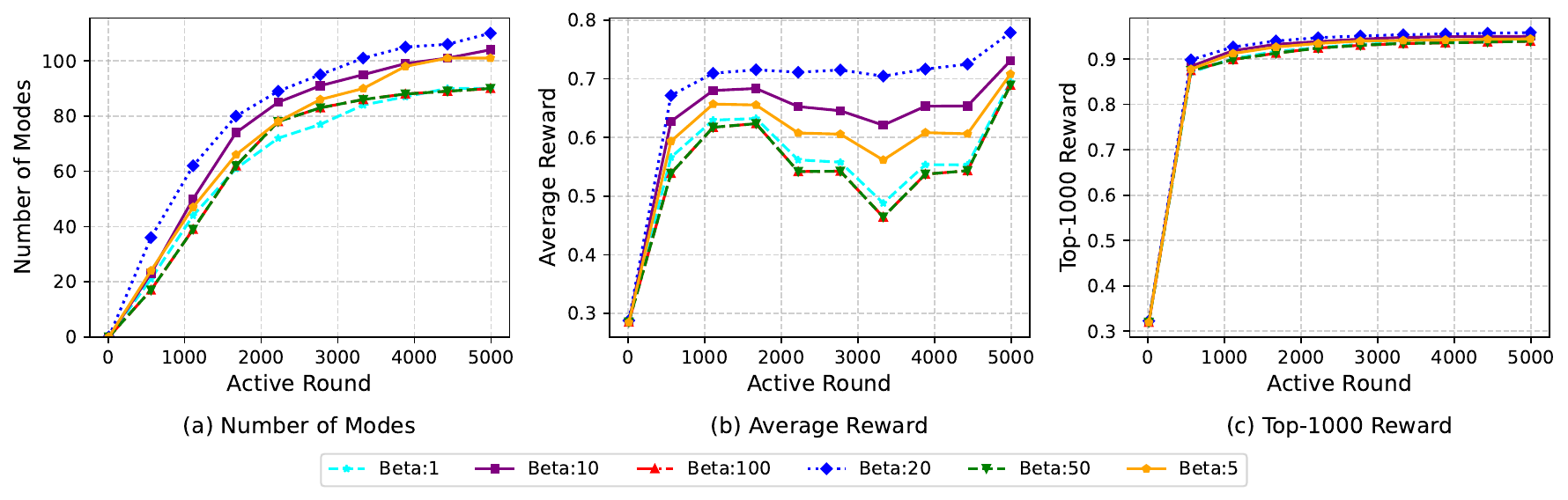}
   \caption{\textbf{Experimental result on RNA-Binding Task 2 with different $\beta$.}}
   \label{beta2}
\end{figure}

In the L14-RNA2 task, while maintaining performance trends consistent with L14-RNA1, the task proved more challenging for mode discovery. All models identified fewer modes compared to L14-RNA1, yet the $\beta = 20$ configuration consistently demonstrated superior performance across all metrics, maintaining its lead over other parameter settings.
\begin{figure} [ht]
   \centering
    \includegraphics[width=\textwidth]{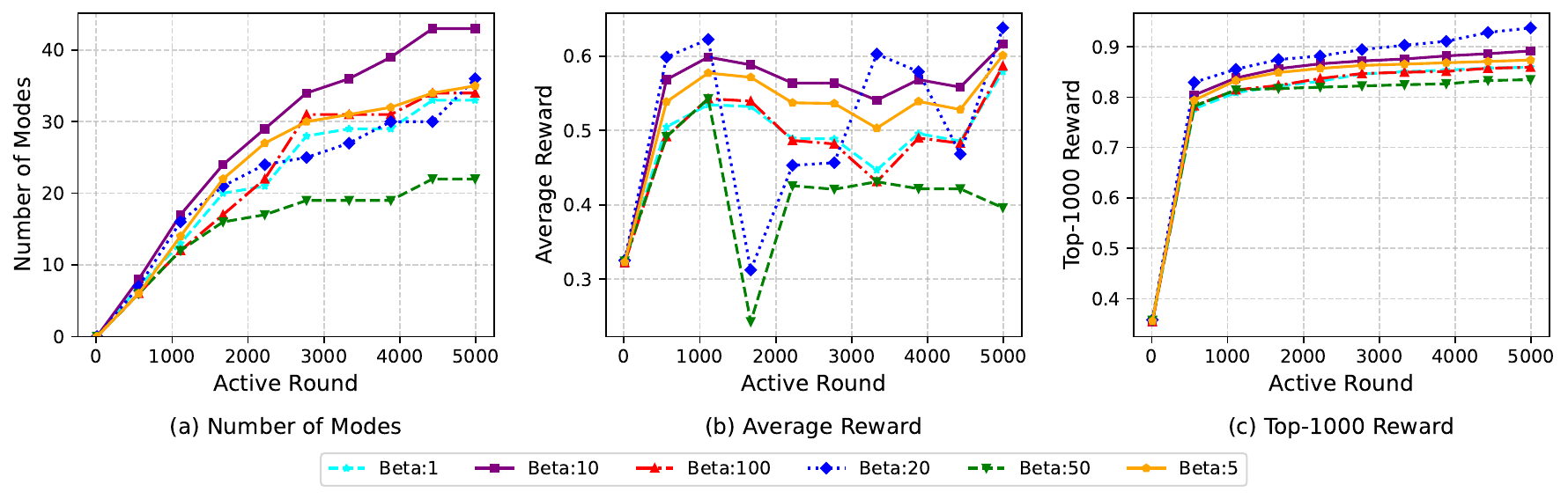}
   \caption{\textbf{Experimental result on RNA-Binding Task 3 with different $\beta$.}}
   \label{beta3}
\end{figure}

L14-RNA3 proves to be the most challenging task, exhibiting a significant decline in both discovered modes and average reward compared to other tasks. In this setting, $\beta = 10$ achieves the best overall performance, whereas $\beta = 20$, while attaining a higher average reward, suffers from instability. The task’s difficulty suggests that using an excessively large $\beta$ may be suboptimal, as it risks mode collapse by overly prioritizing high-reward candidates while neglecting mid-reward solutions. 

\section{Alternative strategies of super arms}
\label{strategies}
Aside from top-K actions, we came up with three alternatives of arm selection: 1) selecting base arms proportional to their scores; 2) selecting base arms randomly; 3) keeping hard pruning but removing CMAB. Here, hard pruning refers to the permanent removal of some fixed actions from the selection pool. Random selection assigns equal selection probabilities to all base arms when constructing the super arm.
\begin{table}[h]
\centering
\caption{Comparison of Different Methods in the molecule generation task.}
\begin{tabular}{lcccc}
\toprule
Method & Modes $R > 7.5$ & Modes $R > 8$ & Top-1000 Reward & Top-1000 Similarity \\
\midrule
Hard-Pruning   & 1069  & 132  & 7.95 & 0.47 \\
Proportional   & 7554  & 1190 & 8.19 & 0.46 \\
Random         & 2179  & 284  & 8.03 & 0.47 \\
Top-K      & 12089 & 2952 & 8.31 & 0.49 \\
\bottomrule
\end{tabular}

\end{table}

When employing a simple hard-pruning approach without the CMAB framework, GFlowNet initially discovers numerous high-reward modes quickly. However, the mode distribution within this constrained subspace is sparse, and after these easily accessible modes are found, the discovery rate drops significantly as the remaining modes become increasingly difficult to identify.

In contrast, a random selection strategy explores the space uniformly by choosing super arms indiscriminately, yielding an expected reward equal to the average across all sub-spaces. This approach achieves performance comparable to the Trajectory Balance (TB) method.

A proportional selection strategy, which chooses arms according to their estimated rewards, naturally outperforms random selection by favoring higher-reward regions. However, it remains less aggressive than the top-K approach. The proportional method discovers fewer total modes yet achieves marginally better performance in top-1000 similarity metrics. The results suggest that while proportional selection maintains better diversity, a more aggressive top-$K$ strategy enables superior overall mode coverage.

% \section{Alternative CMAB algorithms}
% \label{other-cmab}
% Here, we conduct an additional experiment where we replaced CUCB with other CMAB algorithms, the results are listed below:
% \begin{table}[h!]
% \centering
% \caption{Comparison of CMAB methods on reward metrics}
% \begin{tabular}{lccc}
% \toprule
% Method & Modes $R>7.5$ & Modes $R>8$ & Top-K Reward \\
% \midrule
% CUCB & 13074 & 3207 & 8.436 \\
% CTS  & 12220 & 2725 & 8.431 \\
% ESCB & 13568 & 3315 & 8.442 \\
% \bottomrule
% \end{tabular}
% \end{table}

% These results confirm that our framework is not dependent on CUCB. Replacing CUCB with CTS or ESCB yields very similar performance, with ESCB even slightly outperforming CUCB in all metrics. This demonstrates that the improvement comes from the pruning framework itself, not from a particular bandit choice. The framework remains effective across different CMAB algorithms. Regarding non-bandit pruning strategies such as curiosity-driven reinforcement learning or evolutionary search, we note that while related, they involve substantially different design challenges, including well-known issues such as curiosity traps and detachment. These directions fall outside the focus of this paper, which centers specifically on bandit-based pruning within our proposed framework.

\section{Illustration of workflow}
Here, we present a case study to illustrate how the proposed method work in the bit sequence generation task. The set $M$ is constructed by randomly combining substrings derived from the base patterns ${00000000, 11111111, 11110000, 00001111, 00111100}$. For $K = 4$, theoretical analysis reveals that the optimal super arm configuration is $S = {0000, 1111, 1100, 0011}$, which achieves perfect mode identification within its substate space with an average distance of $0.0$ to all target modes.

Initially, the base arm $0000$ gradually gains higher values (as shown in Figure \ref{fig:bit-seq-all}), leading to a suboptimal super arm configuration $S = {0000, 1111, 1100, 0001}$. The UCB mechanism in line 11 of Algorithm \ref{alg:CUCB} then identifies $0011$ as a promising alternative - while its current estimated value is slightly lower than $0001$, its higher uncertainty (due to insufficient exploration) suggests significant potential. This triggers an exploration phase where the algorithm selects $S = {0000, 1111, 1100, 0011}$ as the new super arm. Subsequent evaluations confirm that $0011$ consistently generates higher-quality candidates, and the value update in line 15 of Algorithm ~\ref{alg:CUCB} reinforces its estimate in future rounds.

Notably, even though we know a priori that $S$ is optimal, the CMAB framework continues to explore alternative subspaces with some probability. This characteristic ensures the algorithm maintains the capability to discover potentially better configurations while predominantly exploiting the known optimal solution, effectively balancing the exploration-exploitation trade-off throughout the learning process.

\section{Time and GPU Memory Consumption}
\label{time-memory} 
Table~\ref{tab:timecost} reports a more fine-grained profiling breakdown of runtime and memory consumption in the molecule generation task, comparing vanilla GFlowNet and CBFlowNet under different arm size $K$ ($10^2$, $10^4$), where the parenthetical indicates the scale of bandit bookkeeping in our implementation (larger values imply heavier update workload). Overall, \textbf{training time dominates and is nearly identical across methods} ($\sim$311\,s / 100 rounds), indicating that CBFlowNet introduces minimal overhead in network optimization. The extra cost mainly comes from \textbf{sampling time}, while the \textbf{update time} remains negligible but increases slightly as the setting grows. Memory footprints are comparable: GPU memory increases mildly (about +13--52\,MB) and CPU memory by at most $\sim$850\,MB in our profiling.

\begin{table*}[t]
    \centering
    \scriptsize
    \caption{Profiling breakdown of runtime and memory consumption on Molecule Design (per 100 rounds).}
    \begin{tabular}{lrrrrrr}
    \toprule
    Method & Sampling Time & Training Time & Update Time & Other Time & GPU Memory & CPU Memory \\
    \midrule
    CBFlowNet ($10^2$) & 61.07\,s & 308.90\,s & 0.03\,s & 151.9\,s & 4125.0\,MB & 2226.9\,MB \\
    CBFlowNet ($10^4$) & 61.69\,s & 311.23\,s & 0.15\,s & 153.1\,s & 4148.3\,MB & 3051.2\,MB \\
    % CBFlowNet ($10^6$) & 62.61\,s & 313.61\,s & 0.66\,s & 155.7\,s & 4164.5\,MB & 2301.7\,MB \\
    GFlowNet & 50.70\,s & 310.64\,s & 0.00\,s & 153.7\,s & 4112.4\,MB & 2217.3\,MB \\
    \bottomrule
    \end{tabular}
    \label{tab:timecost}
\end{table*}

% \begin{wrapfigure}{r}{0.5\textwidth}
%     \label{fig:prune}
%     \includegraphics[width=0.5\textwidth]{iclr2026/figures/flower.pdf}
%     \caption{Experiment results on ELBO.}
% \end{wrapfigure}
\begin{figure}[ht]
    \centering
    \includegraphics[width=0.5\linewidth]{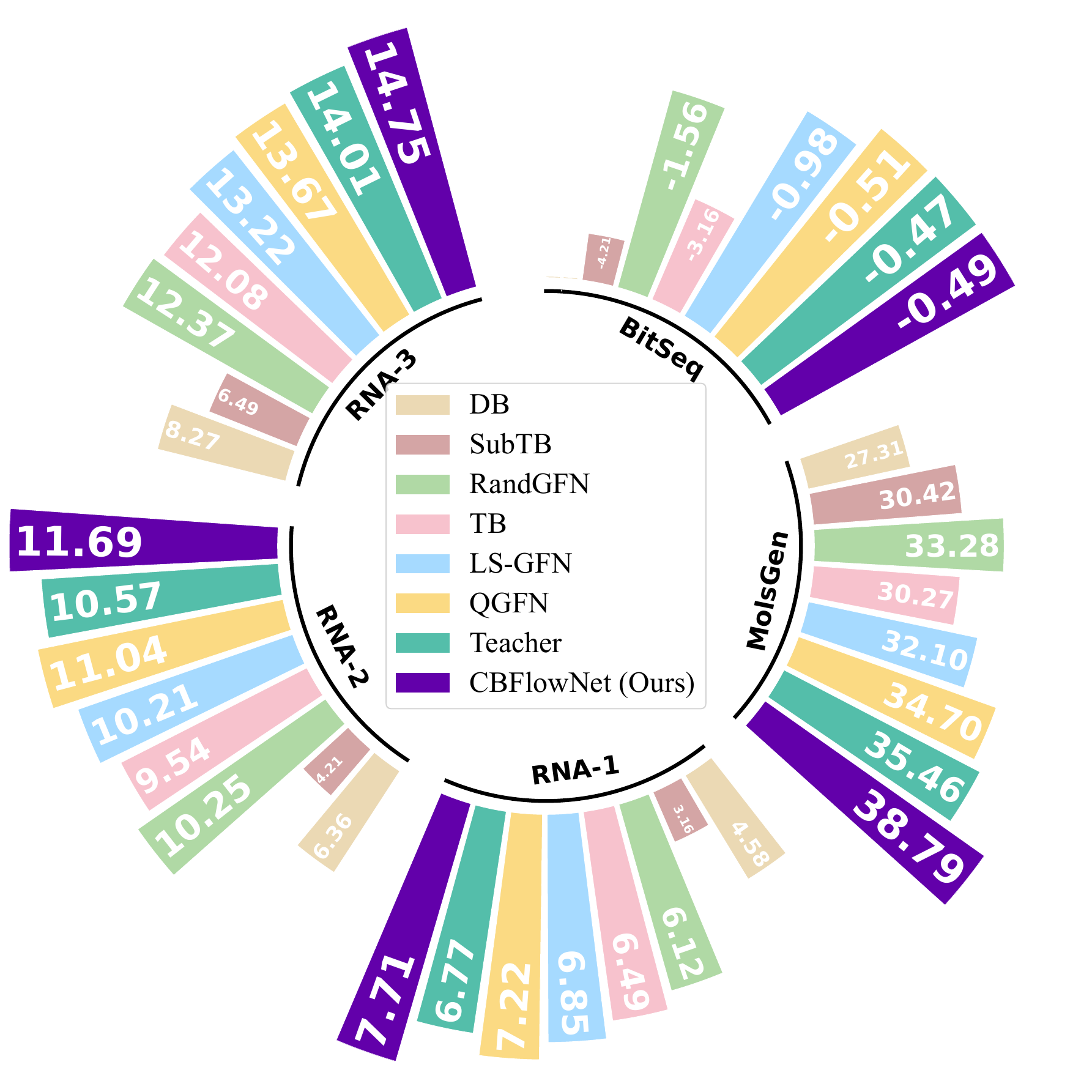}
    \caption{Experiment results on ELBO.}
    \label{fig:elbo}
\end{figure}
\section{Evidence Lower Bound (ELBO)}
\label{ELBO}
We evaluate the goodness of fit to the target distribution using the evidence lower bound (ELBO) introduced by~\citet{kim2024adaptive}. ELBO is estimated by sampling $M$ candidates and averaging the estimated $\log Z$ through a transformed TB objective:
\begin{equation}
Z \prod_{t=1}^n P_F(s_t|s_{t-1}) = F(X)\prod_{t=1}^n P_B(s_{t-1}|s_t).
\end{equation}
The corresponding ELBO is approximated as:
\begin{equation}
    ELBO \approx \frac{1}{M} \sum_{i=1}^M \Bigg( \log R(x_i) 
    + \sum_{t=1}^{n_i} P_B(s_{t-1}|s_t) 
    - \sum_{t=1}^{n_i} P_F(s_t|s_{t-1}) \Bigg),
\end{equation}
where $n_i$ is the length of the trajectory that generates terminal state $x_i$. Results are presented in Figure~\ref{fig:elbo}. The proposed {\modelname} achieves slightly better performance than the baselines, with Teacher remaining the strongest competitor.

\section{Scalability}
We show that the proposed method scales effectively to larger action spaces both theoretically and empirically. When selecting $K$ arms from $N$ total base arms to form a super arm, the accuracy of reward estimates $\hat{\mu}_i$ typically requires more rounds to converge as $N$ grows. However, in practice, $K$ often scales proportionally with $N$—for example, setting $K=0.1N$ (selecting $10\%$ of base arms). Under this scheme, the estimation accuracy of $\hat{\mu}_i$ remains stable with respect to $N$, leading to a fixed convergence rate.  

Moreover, the computation of rewards $\hat{\mu}_i$ can be embedded within the flow-matching updates at negligible cost. Since our algorithm is heuristic, the additional computational burden is minimal, as also indicated in Table~\ref{tab:timecost}.  

To further validate scalability, we experimented with enlarged action spaces. The molecule design task originally contains 105 building blocks with several stems each. By combining two actions into one base arm, the action space increases to $105 \times 105 = 11025$, denoted as {\modelname} (CA). As shown in Table~\ref{tab:performance_modes}, {\modelname} (CA) exhibits comparable or slightly better performance in discovering high-reward modes while incurring only negligible computational overhead.  

\begin{table}[ht]
\centering
\caption{Performance comparison of different methods.}
\resizebox{\textwidth}{!}{
\begin{tabular}{lcccccc}
\toprule
Method & Training Round & Modes $R > 7.5$ & Modes $R > 8$ & Top-1000 Reward & Top-1000 Similarity & Time (s / 100 rounds) \\
\midrule
TB-GFN          & $10^5$ & 1915  & 233  & 8.01 & 0.47 & 24.37 \\
{\modelname}       & $10^5$ & 12089 & 2952 & 8.31 & 0.49 & 27.95 \\
{\modelname} (CA)  & $10^5$ & 12433 & 2520 & 8.29 & 0.50 & 28.73 \\
\bottomrule
\end{tabular}
}
\label{tab:performance_modes}
\end{table}

\section{Experiments with Large Language Model Task}
We also report experiments on a task with dynamic reward distributions (see Section~\ref{limitation}) as a limitation study. Following~\citet{hu2023amortizing}, we considered a subjectivity classification task where each movie review is labeled as either objective or subjective. This task is particularly challenging, as it involves both the E-step and M-step of the EM algorithm, with GFlowNet serving as the inference model in the E-step.  

We adopt the default settings from the public implementation of~\citet{hu2023amortizing}. For fine-tuning, we tested both GPT-2 and GPT-J 6B backbones.  

\begin{table}[ht]
\centering
\caption{Comparison of {\modelname} and GFlowNet under GPT-2 and GPT-J backbones with different sample sizes. Results are reported as mean $\pm$ standard deviation.}
\begin{tabular}{lcccccc}
\toprule
\multirow{2}{*}{Method} 
& \multicolumn{3}{c}{GPT-2} 
& \multicolumn{3}{c}{GPT-J 6B} \\
\cmidrule(lr){2-4} \cmidrule(lr){5-7}
& 10 Samples & 20 Samples & 50 Samples 
& 10 Samples & 20 Samples & 50 Samples \\
\midrule
{\modelname} & 0.59 {\scriptsize$\pm$0.02} & 0.63 {\scriptsize$\pm$0.03} & 0.78 {\scriptsize$\pm$0.02} 
          & 0.71 {\scriptsize$\pm$0.02} & 0.83 {\scriptsize$\pm$0.01} & 0.90 {\scriptsize$\pm$0.01} \\
GFlowNet  & 0.58 {\scriptsize$\pm$0.03} & 0.61 {\scriptsize$\pm$0.02} & 0.75 {\scriptsize$\pm$0.03} 
          & 0.71 {\scriptsize$\pm$0.02} & 0.81 {\scriptsize$\pm$0.02} & 0.87 {\scriptsize$\pm$0.02} \\
\bottomrule
\end{tabular}
\label{tab:llm}
\end{table}

The results in Table~\ref{tab:llm} show that {\modelname} marginally outperforms GFlowNet in test accuracy across all training sample sizes. However, this task highlights a limitation of {\modelname}. The reward of a terminal state $Z$, defined as $p_{LM}(Z, Y \mid X)$, depends on both the label $Y$ and input $X$. For instance, the word \textit{factual} may yield a high reward when the label is ``objective'' but a low reward when the label is ``subjective.'' Thus, the high-reward state space shifts during training, which differs fundamentally from our original setting.  

We categorize base arms into four groups:  
A) high-scoring under ``objective'' and low-scoring under ``subjective'';  
B) high-scoring under ``subjective'' and low-scoring under ``objective'';  
C) consistently high-scoring;  
D) consistently low-scoring.  

Our framework is primarily designed to identify and filter arms of type D while retaining type C. In this dynamic setting, where types A and B fluctuate, we increased the base-arm set size $K$ to ensure sufficient coverage of relevant arms (types A, B, and D).

\end{document}